\documentclass[10pt,twocolumn,letterpaper]{article}

\usepackage{cvpr}
\usepackage{times}
\usepackage{epsfig}
\usepackage{graphicx}
\usepackage{amsmath}
\usepackage{amssymb}
\usepackage{booktabs}       
\usepackage{multirow}
\usepackage{subcaption}

\usepackage{amsthm}
\usepackage{adjustbox}
\usepackage{algorithm}
\usepackage{algorithmic}
\usepackage{balance}

\newtheorem{theorem}{Theorem}
\newtheorem{lemma}[theorem]{Lemma}

\def\sub#1{_{\rm #1}}

\def\eg{{\it e.g.}}

\def\ie{{\it i.e.}}

\usepackage[pagebackref=true,breaklinks=true,letterpaper=true,colorlinks,bookmarks=false]{hyperref}

\cvprfinalcopy 

\def\cvprPaperID{2293} 

\ifcvprfinal\fi
\begin{document}

\title{Decentralized Learning of Generative Adversarial Networks from Non-iid Data}

\author{Ryo Yonetani\\
OMRON SINIC X\\
Tokyo, Japan\\
{\tt\small ryo.yonetani@sinicx.com}
\and
  Tomohiro Takahashi\\
  OMRON Corporation\\
  Kyoto, Japan\\
{\tt\small tomohiro.takahashi.2@omron.com}
\and
  Atsushi Hashimoto\\
  OMRON SINIC X\\
  Tokyo, Japan\\
{\tt\small atsushi.hashimoto@sinicx.com}
\and
  Yoshitaka Ushiku\\
  OMRON SINIC X\\
  Tokyo, Japan\\
{\tt\small yoshitaka.ushiku@sinicx.com}
}

\maketitle

\begin{abstract}
This work addresses a new problem that learns generative adversarial networks (GANs) from multiple data collections that are each i) owned separately by different clients and ii) drawn from a non-identical distribution that comprises different classes. Given such non-iid data as input, we aim to learn a distribution involving all the classes input data can belong to, while keeping the data decentralized in each client storage. Our key contribution to this end is a new decentralized approach for learning GANs from non-iid data called Forgiver-First Update (F2U), which a) asks clients to train an individual discriminator with their own data and b) updates a generator to fool the most `forgiving' discriminators who deem generated samples as the most real. Our theoretical analysis proves that this updating strategy allows the decentralized GAN to achieve a generator's distribution with all the input classes as its global optimum based on f-divergence minimization. Moreover, we propose a relaxed version of F2U called Forgiver-First Aggregation (F2A) that performs well in practice, which adaptively aggregates the discriminators while emphasizing forgiving ones. Our empirical evaluations with image generation tasks demonstrated the effectiveness of our approach over state-of-the-art decentralized learning methods.
\end{abstract}

\section{Introduction}
\label{sec:intro}

Large-scale datasets as well as high-performance computational resources are arguably vital for training many of the state-of-the-art deep learning models. Typically, such datasets have been curated from publicly available data, \eg,~\cite{Sami2016,Russakovsky2015a}, and trained in a single workstation or a well-organized computer cluster. At the same time, increasing attention is being paid to \emph{decentralized learning}, where multiple clients collaboratively utilize their local data resources and computational resources to enable large-scale training, while the data are kept decentralized in each client storage. Unlike much related work focusing on supervised decentralized learning, this work will address an unsupervised task, more specifically, learning generative adversarial networks (GANs)~\cite{Goodfellow2014} from decentralized data.

\begin{figure}[t]
  \begin{center}
    \includegraphics[width=\linewidth]{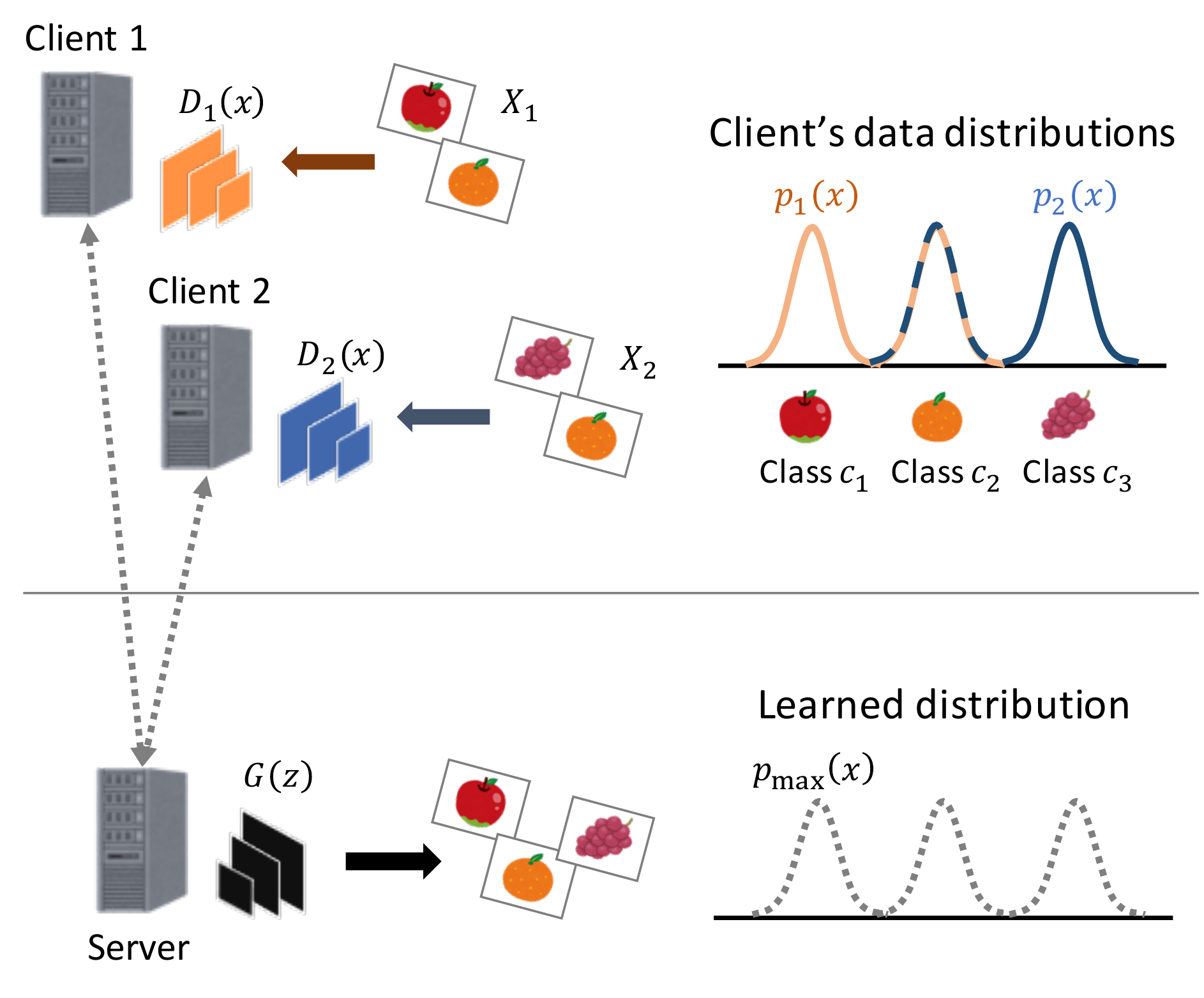}
    \caption{\textcolor{black}{\textbf{Decentralized Learning of GANs}. Individual clients each have private data collections $X_i$ drawn from non-identical distributions $p_i(x)$ that comprise different classes (\eg, classes $\{c_1,c_2\}$ in $X_1$ and $\{c_2,c_3\}$ in $X_2$). By communicating with the clients, the server tries to learn a distribution involving all the input classes (\ie, $\{c_1,c_2,c_3\}$), while keeping $X_1$ and $X_2$ decentralized in each client storage.}}
    \label{fig:teaser}
  \end{center}
\end{figure}

\textcolor{black}{Particularly, we are interested in learning a generative model from multiple image data collections that are each i) owned separately by different clients, and ii) drawn from non-identical data-generating distributions that comprise different classes (\eg, image categories; see also Figure~\ref{fig:teaser}). Given such non-iid data collections as input, we aim to achieve a generative model of a distribution that involves \emph{all the classes input data can belong to, while keeping all the data client-side during the training}. Examples of data domains that would benefit from such unsupervised decentralized learning include life-logging videos~\cite{Chowdhury2016}, biological data~\cite{Ching2018}, and medical data~\cite{Li2005}, which are all obtained independently by multiple individuals or institutions and characterized differently due to geographical conditions or personal preferences. On the one hand, learning from non-iid data resources distributed all over the world would improve the diversity of images the trained model can generate and ultimately enhance various vision-based applications. For example, unsupervised anomaly detection~\cite{Schlegl2017}, if learned from client's data captured under various conditions, can utilizes a distribution of diverse `normal' samples for reliably detecting anomalies. On the other hand, direct access to a collection of data owned by each client may not always be possible due to prohibitively large dataset sizes or privacy concerns (\eg, a living area could be inferred from the statistics of a life-logging video stream taken by a single person). This dilemma makes it hard to aggregate all the client data in a central server, necessitating decentralized learning approaches.}

Nevertheless, it is harder to determine how supervised decentralized learning, which has been extensively studied~\cite{Bonawitz2019,Jiang2017,Lian2017,Wen2017}, can be adopted for learning generative models from decentralized non-iid data. An exception proposed recently is decentralized learning of GANs~\cite{Hardy2018}, which lets each client train an individual discriminator with their own data while asking a central server to update a generator to fool those discriminators. While allowing clients to decentralize their data in their own storage, this approach restricts all client data to be iid, and otherwise has no theoretical guarantee on what distribution will be learned. Consequently, little work has been done on the decentralized learning of generative models from non-iid data, despite the data non-iidness is one of the key properties in a practical setting of decentralized learning~\cite{Mcmahan2017}.

Given this background, we propose a new unsupervised decentralized approach for \emph{learning GANs from non-iid data}, which we refer to as \textbf{Forgiver-First Update (F2U)}. Specifically, given multiple discriminators each trained by different clients with non-identical distributions, hereafter $p_i(x)$, F2U allows a generator to learn $p\sub{max}(x)=\frac{1}{Z}\max_ip_i(x)$ ($Z$ is the normalizing constant) that comprises all the input classes including rare ones observed only by a small fraction of the clients as well as common ones shared by many. Our theoretical analysis based on $f$-divergence minimization proves that $p\sub{max}(x)$ can be achieved as the decentralized GAN's global optimum by letting the generator fool the most `forgiving' discriminators for each generated sample, who deemed the sample as the most real and closest to what they own.

Moreover, we present a relaxed version of F2U called \textbf{Forgiver-First Aggregation (F2A)}. Instead of selecting the most forgiving discriminators, F2A adaptively aggregates judgments of discriminators made to generated samples, while emphasizing those from more forgiving ones, and updates the generator with the aggregated judgments. While sacrificing the theoretical guarantee, F2A often performs better than F2U in practice. Technically, the adaptive aggregation is done by a regularized weighted averaging whose weights are also updated via back-propagation, allowing the generator to better capture the non-iidness of input data. 

We empirically evaluated our approach with image generation tasks on several public image datasets. The experimental results demonstrate that the decentralized GANs trained with F2U and F2A clearly outperformed several state-of-the-art approaches~\cite{Durugkar2016,Hardy2018}.


\subsection*{Related Work}

\textcolor{black}{In this work, we are interested in learning GANs from decentralized data that are owned separately by multiple clients. This problem setting gives a different assumption about data distributions compared to a standard distributed learning paradigm based on a data parallelism scheme. On the one hand, data-parallel distributed learning typically partitions a large-scale data into smaller mini-batches and distributes them to multiple client machines, so that each mini-batch implicitly follows the same data-generating distribution (see also \cite{Tal2018,Lian2017}). On the other hand, the data that \emph{the clients already own} can be highly dissimilar to each other depending on how they were collected, which we refer to as data non-iidness. The most relevant approach is federated learning~\cite{Bonawitz2019,Mcmahan2017} that addressed the problem of learning from such non-iid client data. More recent work has tried to make federated learning more communication efficient~\cite{Jeong2018,Konecny2016, Lin2017}, secure~\cite{Bagdasaryan2018,Bonawitz2017}, and applicable to a practical wireless setting~\cite{Giannakis2016,Wang2018}, but under the setting of standard supervised learning. As we stated earlier, one exception presented recently is decentralized learning of GANs~\cite{Hardy2018}, which however worked only when client data were iid. To the best of our knowledge, our work is the first to address the problem of unsupervised decentralized learning from non-iid data collections.}

In terms of the formulation of GANs, recent work has also tried to involve multiple discriminators and/or multiple generators. The motivations behind such works are, however, not to enable decentralized learning for multi-client data but to stabilize the training process~\cite{Chavdarova2018,Durugkar2016}, to avoid mode collapses~\cite{Ghosh2018,Nguyen2017}, or to model multi-domain data~\cite{Choi2018,Liu2016,Zhu2017}, with the centralized setting.

\section{Preliminaries}
\label{sec:proposed}

\paragraph{Problem Setting.}
Consider $N$ clients who each have their own private data collection $X_i=\{x\mid x\sim p_i\}$ drawn from a hidden, non-identical data-generating distribution $p_i(x)$ that comprises different disjoint classes (\eg, classes $\{c_1, c_2\}$ in $X_1$ and $\{c_2,c_3\}$ in $X_2$ as shown in Figure~\ref{fig:teaser}). Given non-iid data collections $\mathcal{X}=\{X_i\mid i=1,\dots,N\}$ as input, we address the problem of learning a GAN with the generator's distribution $p\sub{g}(x)$ given by:
\begin{eqnarray}
p\sub{max}(x)&=&\frac{1}{Z}\max_ip_i(x),\\
Z&=&\int_x\max_ip_i(x)\mathrm{d}x,
\end{eqnarray}
\emph{while keeping $\mathcal{X}$ decentralized such that each $X_i$ is available only for the $i$-th client}. The distribution $p\sub{max}(x)$ contains all the classes that $\mathcal{P}=\{p_1,\dots,p_N\}$ collectively have (\ie, $\{c_1, c_2, c_3\}$ in Figure~\ref{fig:teaser})\footnote{More generally, consider a set of $K$ disjoint classes $\Pi=\left\{c_1, c_2,\dots,c_K\right\}$ each of which has its own data-generating distribution $\rho_k(x)\in [0, 1],\;\int_x\rho_k(x)\textrm{d}x=1$. Then we denote by $\Pi_i\subset \Pi$, a subset of classes that $p_i(x)$ comprises, and define $p_i(x)=\sum_{k=1}^{K}w_i(k)\rho_k(x)$ where $w_i(k)\in[0,1]$ is a class prior that satisfies i) $\sum_k w_i(k)=1$, and ii) $w_i(k)>0$ if $c_k \in\Pi_i$ and $w_i(k)=0$ otherwise. Because we assume the classes to be disjoint, $\max_ip_i(x)=\sum_k\max_iw_i(k)\rho_k(x)=\sum_kw\sub{max}(k)\rho_k(x)$ where $w\sub{max}(k)=\max_iw_i(k)>0$ if $\exists \Pi_i,\;c_k\in\Pi_i$ and $w\sub{max}(k)=0$ otherwise. Namely, $p\sub{max}$ comprises all the classes $\mathcal{P}$ collectively have.}. Compared to other possible distributions that could be learned from $\mathcal{P}$, such as $\frac{1}{N}\sum_i p_i(x)$ and $\frac{1}{Z'}\min_i p_i(x)$ where $Z'=\int_x\min_ip_i(x)\textrm{d}x$, learning a generator for $p\sub{max}(x)$ is advantageous when we aim to generate diverse samples including those of rare classes observed only by a small fraction of $\mathcal{P}$ as well as common ones shared by many. 

\paragraph{Generative Adversarial Networks.}
As a preliminary, let us briefly introduce a formulation for training GANs with the centralized setting. Namely, we assume that data sample $x$ drawn from data-generating distribution $p\sub{data}(x)$ can be accessed in a single place without any restriction. Typically, GANs consist of a generator $G$ and a discriminator $D$ which work as follows: $G$ takes as input a noise vector $z$ drawn from a normal distribution $p\sub{z}(z)$ to generate a realistic sample $G(z)$, which is ideally a generative model of a distribution $p\sub{data}(x)$. $D$ receives either real samples $x \sim p\sub{data}$ or generated ones $G(z)$ to discriminate them.

Training of GANs proceeds based on the competition between $G$ and $D$; they are coupled and updated by minimizing the following two objective functions alternately:
\begin{eqnarray}
    \mathcal{L}_D=\mathop{\mathbb{E}}_{x\sim p\sub{data}}\left[l\left(D(x), y\sub{r}\right)\right]  + \mathop{\mathbb{E}}_{z\sim p\sub{z}}\left[l\left(D(G(z)), y\sub{f}\right)\right],
    \label{eq:centralized_d}\\
    \mathcal{L}\sub{G}=\mathop{\mathbb{E}}_{x\sim p\sub{data}}\left[l\left(D(x), y\sub{r'}\right)\right]+\mathop{\mathbb{E}}_{z\sim p\sub{z}}\left[l\left(D(G(z)), y\sub{r'}\right)\right],
    \label{eq:centralized_g}
\end{eqnarray}
where $\mathcal{L}\sub{G}$ can also be represented by
\begin{equation}
   \mathcal{L}\sub{G}=\mathop{\mathbb{E}}_{x\sim p\sub{data}}\left[l\left(D(x), y\sub{r'}\right)\right]+\mathop{\mathbb{E}}_{x\sim p\sub{g}}\left[l\left(D(x), y\sub{r'}\right)\right],
\end{equation}
with generator's distribution $p\sub{g}$. $l$ is defined differently for the choice of loss functions such as binary cross entropy \cite{Goodfellow2014} and mean squared error \cite{Mao2017}, measuring how judgments of $D$ are different from labels $y\sub{r}, y\sub{r'}$ (real) or $y\sub{f}$ (fake). Given mini-batches of $x$ and $z$, $D$ is updated with Eq.~(\ref{eq:centralized_d}) while $G$ fixed to detect generated samples more accurately. $G$ is updated with Eq.~(\ref{eq:centralized_g}) while fixing $D$ to generate more realistic samples that are more likely to fool $D$.

\paragraph{Decentralized Learning of GANs.}
Now we consider our main problem: decentralized learning of GANs from non-iid data collections. Following the basic idea of Multi-Discriminator GAN (MD-GAN) proposed in \cite{Hardy2018}, we consider the existence of a \emph{server} that collaborates with the clients to learn GANs. Under the decentralized setting, it is reasonable to ask each client to learn its own discriminator $D_i$ with $X_i$ and to inform the server how $D_i$ judges generated samples $G(z)$ instead of directly sharing $X_i$. Then the server can update the generator $G$ to fool $D_1,\dots,D_N$.

Within this approach, our main technical question is that, \emph{``how can judgments $D_i(G(z))$ be taken into account by $G$ to learn $p\sub{max}(x)$, especially when $p_i(x)$s are non-identical?''} Unfortunately, this question remains unsolved in MD-GAN. As we introduced in Section~\ref{sec:intro}, it assumes all of the input data collections to be iid, and has no theoretical guarantee on what will be learned as the generator's distribution $p\sub{g}$ otherwise. Other relevant GANs using multiple discriminators, such as Generative Multi-Adversarial Networks (GMAN)~\cite{Durugkar2016}, also assume all the data samples to be drawn from the same distribution, making it difficult to address the data non-iidness.


\section{Proposed Approach}
In this section, we present i) \textbf{Forgiver-First Update (F2U)} that is proven to achieve $p\sub{g}=p\sub{max}$ as the global optimum of decentralized GAN with multiple discriminators; and ii) \textbf{Forgiver-First Aggregation (F2A)} that can work well in practice.

\subsection{Forgiver-First Update}
\label{subsec:F2U}

\begin{figure}[t]
  \begin{center}
    \includegraphics[width=\linewidth]{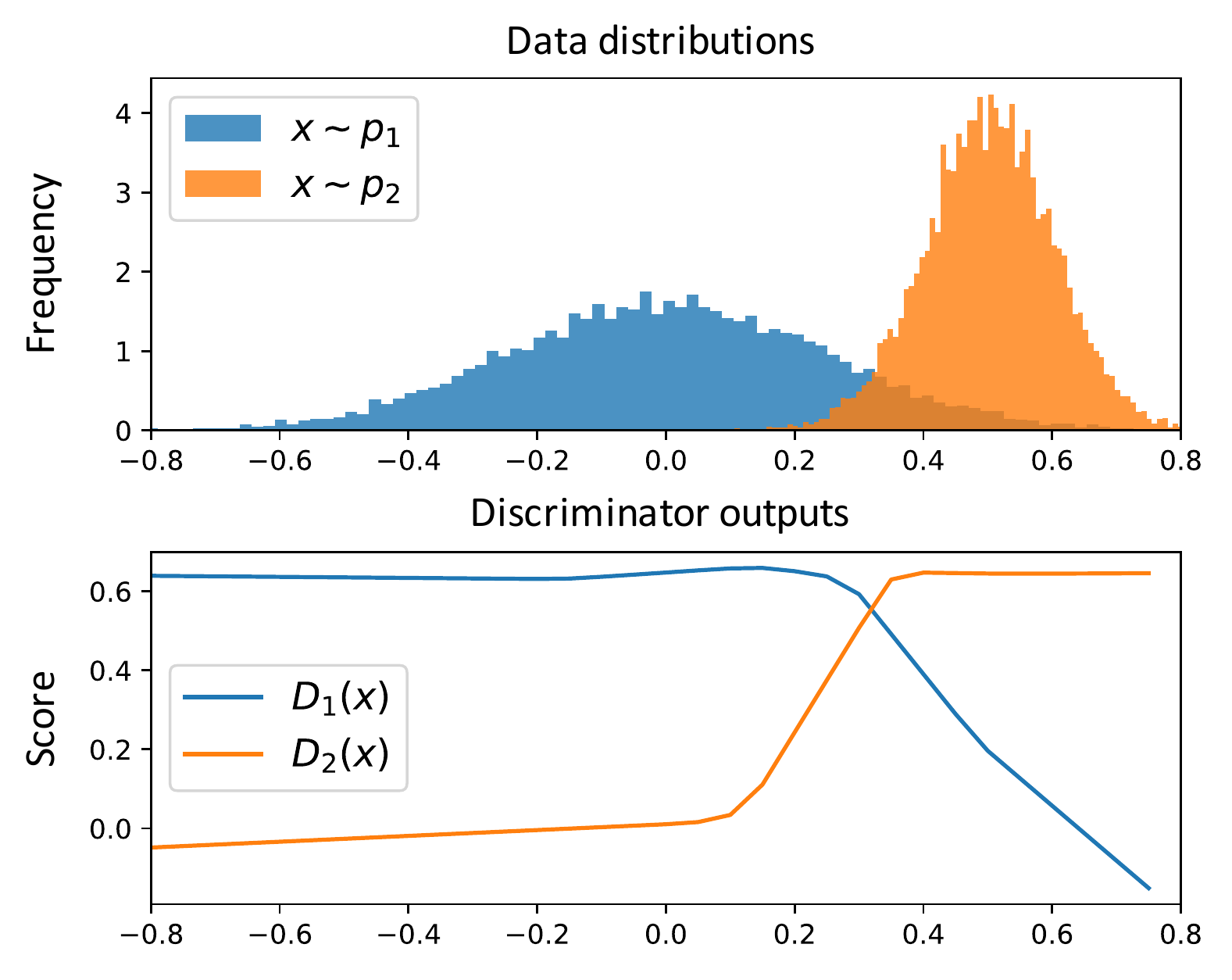}
    \caption{\textcolor{black}{\textbf{Two-Class Non-iid Example.} Client data distributions (top) and outputs of the corresponding discriminators trained with our approach (bottom). See Figure~\ref{fig:toy} for the learned distribution.}}
    \label{fig:overlap}
  \end{center}
\end{figure}

As its name implies, F2U asks generator $G$ to be updated against the discriminator who gives the most forgiving judgment, which can be described simply by:
\begin{equation}
    D\sub{max}(x)=\max_i D_i(x).
\end{equation}
To better understand this approach, Figure~\ref{fig:overlap} illustrates an example of two-client non-iid data each drawn from a distinct Gaussian distribution (top) and outputs of the corresponding client-wise discriminators trained with our approach (bottom). Our key insight is that, when client data are highly dissimilar between each other, the discriminators will judge each sample differently depending on where the sample is, such as $D_1(x)\gg D_2(x)$ if $p_1(x) \gg p_2(x)$ and $D_1(x) \ll D_2(x)$ if $p_1(x)\ll p_2(x)$. Accordingly, selecting $\max_i D_i(x)$ to update $G$ intuitively means selecting $\max_ip_i(x)$ as the data distribution that $G$ will learn.

Below we prove that our decentralized GAN achieves $p\sub{g}=p\sub{max}$ as the global optimum if $G$ is updated with $\max_iD_i(x)$. Here we focus on a typical setting of least-square GANs (LSGANs)~\cite{Mao2017} that will be used in our experiments, where $l$ is defined by the mean-squared error, $y\sub{r}, y\sub{r'}=1$, and $y\sub{f}=0$ (another case with the standard GAN is also present in the supplementary material).

As shown in \cite{Mao2017}, the optimal discriminator given data-generating distribution $p_i(x)$ and generator's distribution $p\sub{g}(x)$ and $y\sub{r}=1, y\sub{f}=0$ is:
\begin{equation}
    D_i^*(x)=\frac{p_i(x)}{p_i(x) + p\sub{g}(x)}.
    \label{eq:optimal_d}
\end{equation}
This leads to the following lemma.
\begin{lemma}
If each $D_i(x)$ is trained optimally from data-generating distribution $p_i(x)$ and generator's distribution $p\sub{g}(x)$, $D^*\sub{max}(x)=\max_i D^*_i(x)$ can be regarded as the optimal discriminator trained from $p\sub{max}(x)$, \ie, $D^*\sub{max}(x)=\frac{p\sub{max}(x)}{p\sub{max}(x) +\alpha p\sub{g}(x)}$ where $\alpha$ is a positive constant.
\label{proof:lemma1}
\end{lemma}
\begin{proof}
Eq.~(\ref{eq:optimal_d}) can be represented by $D^*_i(x)=1-\frac{p\sub{g}(x)}{p_i(x)+p\sub{g}(x)}$. By fixing $x$ and regarding $p\sub{g}(x)$ as a positive constant, we see that $D^*_i(x)$ monotonically increases with $p_i(x)$ within $p_i(x)\in[0, 1]$. Thus, 
\begin{eqnarray}
    D^*\sub{max}(x)&=&\max_i D^*_i(x) = \frac{\max_ip_i(x)}{\max_ip_i(x)+p\sub{g}(x)}\\
    &=&\frac{p\sub{max}(x)}{p\sub{max}(x) +\alpha p\sub{g}(x)},
\end{eqnarray}
where $\alpha=\frac{1}{Z}=(\int_x \max_i p_i(x)\textrm{d}x)^{-1} > 0$.
\end{proof}

Then, by substituting $D(x)=D^*\sub{max}(x)$ and $y\sub{r'}=1$ for the objective function $\mathcal{L}\sub{G}$ in Eq.~(\ref{eq:centralized_g}), and by letting $l$ be the mean-squared error as done in \cite{Mao2017}, we obtain:
\begin{eqnarray}
    \mathcal{L}\sub{G}&=&\frac{1}{2}\left\{\mathop{\mathbb{E}}_{x\sim p\sub{max}}\left[(D^*\sub{max}(x) - 1)^2\right] \right.\\
    &&\left.+ \mathop{\mathbb{E}}_{x\sim p\sub{g}}\left[(D^*\sub{max}(x) - 1)^2\right]\right\}\\
    &=&\frac{1}{2}\int_x\frac{(p\sub{max}(x)+p\sub{g}(x))\alpha^2 p\sub{g}^2(x)}{(p\sub{max}(x)+\alpha p\sub{g}(x))^2}\mathrm{d}x.
    \label{eq:objective}
\end{eqnarray}
\begin{theorem}
The global minimum of $\mathcal{L}\sub{G}$ given $D^*\sub{max}(x)$ is achieved if and only if $p_g=p\sub{max}$.
\label{proof:theorem2}
\end{theorem}
\begin{proof}
Importantly, the theoretical result in \cite{Mao2017} based on the minimization of Pearson $\chi^2$ divergence is not directly applicable here because $\alpha=\frac{1}{Z}$ in Eq.~(\ref{eq:objective}) is fixed but unknown in practice. To explicitly deal with $\alpha$ in the divergence minimization, we introduce the following function $f$:
\begin{equation}
   f(x)=\frac{(x+1)\alpha^2 x^2}{(1 + \alpha x)^2} - \frac{2\alpha^2}{(1+\alpha)^2},
\end{equation}
where $f(1)=0$, continuous and convex for $x\geq 0$ (see the supplementary material for more detail). This function can then be used to define the $f$-divergence below:
\begin{eqnarray}
    \mathcal{D}_f(p||q) &=& \int_x q(x)f\left(\frac{p(x)}{q(x)}\right)\mathrm{d}x\\
    &=&\int_x\frac{(q(x)+p(x))\alpha^2p^2(x)}{(q(x)+\alpha p(x))^2}\mathrm{d}x + C,
\end{eqnarray}
where $C=-\frac{2\alpha^2}{(1+\alpha)^2}$ is a constant. This $f$-divergence $\mathcal{D}_f$ is non-negative and becomes zero if and only if $p=q$. Finally, $\mathcal{L}\sub{G}$ in Eq.~(\ref{eq:objective}) can be rearranged with $\mathcal{D}_f$ as follows:
\begin{equation}
    \mathcal{L}\sub{G}=\frac{1}{2}\mathcal{D}_f(p\sub{g}||p\sub{max}) - \frac{1}{2}C.
    \label{eq:objective_final}
\end{equation}
From Eq.~(\ref{eq:objective_final}), $\mathcal{L}\sub{G}$ reaches the global minimum if and only if $p\sub{g}=p\sub{max}$.
\end{proof}

\subsection{Forgiver-First Aggregation}
\label{subsec:F2A}

While F2U has a theoretical guarantee to achieve $p\sub{g}=p\sub{max}$ as the global optimum, we rarely obtain optimal discriminators $D^*_i(x)$ in practice. Moreover, \cite{Durugkar2016} shows that involving many discriminators, instead of selecting one of them, can accelerate the training process. Therefore, we propose F2A that \emph{aggregates} $D_i(x)$s while emphasizing more forgiving ones as follows:
\begin{equation}
D\sub{agg}(x) = \sum_i S_\lambda(x)D_i(x),
\label{eq:f2a_forward}
\end{equation}
where
\begin{equation}
    S_\lambda(x) = \frac{\exp(\lambda D_i(x))}{\sum_j\exp(\lambda D_j(x))}.
\end{equation}
Here, $\lambda\geq 0$ is a parameter allowing us to take different aggregation strategies to better adapt given client data. When $\lambda$ becomes larger, $D\sub{agg}(x)$ will converge to $D\sub{max}(x)$, which would benefit those cases where client data are highly non-iid and severely overlapping. In contrast, when $\lambda$ is nearly $0$, $D\sub{agg}(x)$ will become just the average of $D_i(x)$, which would work well when the client data are iid and significantly overlapping.

Moreover, we update $\lambda$ adaptively with $G$. This makes it unnecessary to manually try multiple choices of $\lambda$ to find better ones based on how non-iid client data are. Specifically, we augment $\mathcal{L}\sub{G}$ in Eq.~(\ref{eq:centralized_g}) to introduce the following regularized objective to update $G$:
\begin{eqnarray}
\mathcal{L}'\sub{G}&=&\mathop{\mathbb{E}}_{z \sim p\sub{z}}\left[l\left(D\sub{agg}(G(z)), y\sub{r'}\right)\right] + \beta\lambda^2 \\
&=& \mathop{\mathbb{E}}_{x \sim p\sub{g}}\left[l\left(D\sub{agg}(x), y\sub{r'}\right)\right] + \beta\lambda^2, 
\label{eq:decentralized_g}
\end{eqnarray}
where we omit the term $\mathop{\mathbb{E}}_{x\sim p\sub{data}}\left[l\left(D(x), y\sub{r'}\right)\right]$ in Eq.~(\ref{eq:centralized_g}) as it does not contain $G$. By computing the gradient of $\mathcal{L}'\sub{G}$ with respect to $\lambda$, we obtain:
\begin{equation}
\frac{\partial \mathcal{L}'\sub{G}}{\partial \lambda} = \mathbb{E}_{x}\left[\frac{\partial l}{\partial D\sub{agg}(x)}\frac{\partial D\sub{agg}(x)}{\partial \lambda}\right] + 2\beta \lambda,\label{eq:f2a_bp_lambda0}\\
\end{equation}
\begin{equation}
\small
\frac{\partial D\sub{agg}(x)}{\partial \lambda} = \sum_i S_\lambda(x) D_i(x)^2 - \left(\sum_i S_\lambda(x) D_i(x)\right)^2.
\label{eq:f2a_bp_lambda}
\end{equation}
Here, $\frac{\partial l}{\partial D\sub{agg}(x)}$ is the original loss gradient measured on a single aggregated judgment $D\sub{agg}(x)$, which is multiplied by $\frac{\partial D\sub{agg}(x)}{\partial \lambda}$ that can be viewed as the variance of $D_i(x)$ weighted by $S_\lambda(x)$. Updating $\lambda$ by gradient descent with $\frac{\partial \mathcal{L}'\sub{G}}{\partial \lambda}$ is therefore reasonable because $\lambda$ will be increased when input data collections are non-iid and making $D_i(x)$s diverse, and be decreased otherwise.

For updating generator $G$, let us denote the parameters of $G$ by $\theta\sub{g}$. While simplifying formal notations of chain rules, the loss gradient with respect to $\theta\sub{g}$ is derived as follows:
\begin{equation}
\small
    \frac{\partial \mathcal{L}'\sub{G}}{\partial \theta\sub{g}} = \mathbb{E}_{x}\left[\frac{\partial l}{\partial D\sub{agg}(x)}\left(\sum_i\frac{\partial D\sub{agg}(x)}{\partial D_i(x)}\frac{\partial D_i(x)}{\partial x}\right)\frac{\partial x}{\partial \theta\sub{g}}\right],
\label{eq:f2a_bp_g}
\end{equation}
where $\frac{\partial D\sub{agg}(x)}{\partial D_i(x)}$ can further be decomposed to:
\begin{equation}
\frac{\partial D\sub{agg}(x)}{\partial D_i(x)} = S_\lambda(x) + \lambda D_i(x) S_\lambda(x)(1-S_\lambda(x)).
\label{eq:ds}
\end{equation}
When $\lambda$ is small, it makes the first term $S_\lambda(x)$ dominant and treats $\frac{\partial D_i(x)}{\partial x}$ as equal when updating $G$. In contrast, when $\lambda$ becomes large, it emphasizes $\frac{\partial D_i(x)}{\partial x}$ with larger $D_i(x)$ and encourages $G$ to fool more forgiving $D_i(x)$.

\textcolor{black}{Note that the adaptive aggregation presented above is partly inspired by \cite{Durugkar2016}. It aims at the stabilizing the learning process by aggregating the losses computed with each discriminator, \ie, $l(D_i(G(z)), y\sub{r'})$, with a trainable softmax function so that \emph{discriminators with higher losses are emphasized more} when updating the generator. By contrast, the proposed F2A can be viewed as emphasizing discriminators with lower losses, by adaptively aggregating discriminator's responses instead of their losses. This formulation enables the update rule of $\lambda$ based on Eq.~(\ref{eq:f2a_bp_lambda}), allowing our approach to learn a variety of data distributions with unknown non-iidness, while \cite{Durugkar2016} assumes all the discriminators to be trained on the identical data distribution.}

\section{Experiments}
We empirically evaluate the decentralized GANs with F2U and F2A on image generation tasks using decentralized versions of public image datasets. \textcolor{black}{Note that our goal here is to evaluate how our aggregation algorithms work well on decentralized non-iid data. Therefore, we implemented all of the generator, discriminators, and client data in a single workstation for the simulation. Similar to other relevant work~\cite{Mcmahan2017}, more practical aspects of decentralized learning frameworks such as their communication throughput are beyond the scope of our experimental evaluation.}

\subsection{Implementation Details}
As a backbone model, we implemented a variant of LSGANs~\cite{Mao2017} with a DCGAN~\cite{Radford2015}-based architecture, which had spectral normalization~\cite{Miyato2018} instead of batch normalization~\cite{Ioffe2015} in discriminators (see our supplementary material for detail). Deeper models such as ones with residual blocks~\cite{He2016} and other sophisticated techniques such as gradient penalty~\cite{Gulrajani2017} would provide higher performance but were not used in this paper, because our focus is not to obtain the best possible performance but to investigate if GANs trained with our approaches could perform well under the decentralized non-iid setting.

To decentralize the backbone model, the discriminator presented above was instantiated $N$ times and initialized independently. The aggregation parameter $\lambda\geq 0$ for F2A was implemented by a one-channel fully connected layer with hidden trainable parameter $\lambda^*$ followed by ReLU activation, which was able to output $\lambda$ by receiving $1$ as input, \ie, $\lambda=\mathrm{ReLU}(\lambda^*\cdot 1)$. $\lambda^*$ was initialized by $0.1$, and the regularization strength $\beta$ was set to $0.1$. Both the generator and the discriminators were trained using Adam~\cite{Kingma2014} with learning rate $\eta=0.0002$, $\alpha=0.5, \beta=0.999$. All the implementations were done with Keras and evaluated on NVIDIA Tesla V100.

\subsection{Baseline Methods}
F2U and F2A consist of multiple discriminators to train a single generator. We chose the following state-of-the-art GANs as baseline methods. For all of the methods, we used the same model architecture and optimization strategy to ensure fair comparisons.

\begin{itemize}
\item\textbf{Multi-Discriminator Generative Adversarial Networks (MD-GAN)}\cite{Hardy2018} is a pioneering attempt to decentralize GAN training that a) asks each client to train an individual discriminator with their own data and b) updates the generator to fool those multiple discriminators. The generator is updated by applying loss gradients computed by each discriminator in turn.
\item \textbf{Generative Multi-Adversarial Networks (GMAN)} \cite{Durugkar2016} is another recent work that introduced multiple discriminators that were however trained from the identical data distribution as discussed in Section~\ref{subsec:F2A}. We evaluated two variants of GMAN proposed in the paper: \textbf{GMAN*} that updated the adaptive aggregation parameter $\lambda$ via back-propagation, and \textbf{GMAN-0} that used fixed $\lambda=0$.
\end{itemize}
\textcolor{black}{We also evaluated a backbone GAN trained with the centralized setting. We stress here that it is not possible to evaluate other decentralized learning methods such as~\cite{Lian2017,Mcmahan2017} because they focus exclusively on supervised learning. For example, \cite{Hardy2018} tried to adapt federated learning~\cite{Mcmahan2017} to the problem of learning GANs from decentralized data, which however failed to work as shown in their results.}

\subsection{Data and Preprocessing}
\textcolor{black}{We used the training split of MNIST (60,000 samples) for analyzing how the proposed F2U and F2A performed in detail. Moreover, we compared our approaches against the baselines on the training split of Fashion MNIST (60,000 samples)~\cite{Xiao2017} and CINIC-10 (270,000 samples)~\cite{Darlow2018}. CINIC-10 is particularly a challenging dataset as it is more than four times larger than MNIST and Fashion MNIST. Also, the samples in CINIC-10 are diverse as they are drawn from CIFAR-10 and ImageNet.}

All the datasets above comprised 10 different classes. We set $N=5$ and split the dataset into five subsets (\ie, $X_1, X_2,X_3,X_4,X_5$ as described below) with the following conditions such that the original data distribution $p\sub{data}$ can be regarded as $p\sub{data}=p\sub{max}$:
\begin{itemize}
\item\textbf{Non-Overlapping (Non-Ovl)}: The subsets $X_1,X_2,X_3,X_4,X_5$ respectively contained the images of $\{0, 1\}$, $\{2, 3\}$, $\{4, 5\}$, $\{6, 7\}$, $\{8, 9\}$-th classes, standing for the most challenging condition.
\item\textbf{Moderately Overlapping (Mod-Ovl)}: The subsets $X_1,X_2,X_3,X_4,X_5$ respectively contained the images of $\{0, 1, 2, 3\}$, $\{2, 3, 4, 5\}$, $\{4, 5, 6, 7\}$, $\{6, 7, 8, 9\}$, $\{8, 9, 0, 1\}$-th classes.
\item\textbf{Fully Overlapping (Full-Ovl)}: All the subsets contained all the classes equally, though such iid cases were not of our main focus.
\end{itemize}

\subsection{Evaluation Process and Metric}
We chose the Fr\'echet Inception Distance (FID)~\cite{Heusel2017} as our evaluation metric. As discussed in \cite{Lucic2018}, FID is sensitive to mode dropping; \ie, it degrades when some of the classes contained in the original dataset are missing in generated data, and thus serves as a suitable metric in this work to see if all the different classes client data had were learned successfully. In our experiments, we randomly sampled 10,000 images from both of the training data and trained generators to compute FID scores.

Importantly, we found that the choices of hyperparameters for training GANs, such as a mini-batch size and the number of iterations, affected FID scores greatly and differently for each method, which was also discussed in \cite{Lucic2018}. Instead of picking out one specific choice of hyperparameters, we tested each method with the combinations of mini-batch sizes $\{32, 64\}$ and the number of iterations $\{25000, 50000\}$ for MNIST and Fashion MNIST and $\{50000, 100000, 150000, 200000, 250000, 300000\}$ for CINIC-10, and reported the median, minimum, and maximum FID scores of those hyparparameter combinations. Each combination was tested once with fixed random seeds.

\subsection{Results}

\paragraph{Comparison with Baselines}
Tables~\ref{tab:MNIST}, \ref{tab:FMNIST}, and \ref{tab:CINIC10} list the FID scores. For all the datasets we tested, we found that i) when client data were non-overlapping or moderately overlapping, F2U or F2A clearly outperformed the baselines; and ii) when the data were fully overlapping, MD-GAN worked best but yet the other approaches including F2A (and sometimes F2U) performed comparably well against backbone GANs trained with the centralized setting.
Moreover, we visualize how the aggregation parameter $\lambda$ changed over iterations in Figure~\ref{fig:MNIST_lambda}. We confirmed that, when client data were non-overlapping or moderately overlapping, $\lambda$ increased greatly until it was saturated by regularization. In contrast, $\lambda$ increased less when the data were fully overlapping.

\begin{table}[t]
\caption{\textbf{FID Scores on MNIST} in the form of \texttt{median (min - max)} across multiple hyperparameter combinations. The backbone GAN trained with the centralized setting achieved 19.37 (16.14 - 23.01).}
\label{tab:MNIST}
\centering
\adjustbox{max width=\linewidth}{
\begin{tabular}{lccc}
\toprule
 Methods                 & Non-Ovl          & Mod-Ovl         & Full-Ovl        \\
\midrule
 \multirow{2}{*}{MD-GAN} & 38.42                     & 34.33                    & \textbf{11.90}                    \\
                         & {\small (36.49 - 39.42)}  & {\small (26.76 - 41.13)} & \textbf{\small (10.85 - 15.13)} \\
 \multirow{2}{*}{GMAN*}  & 67.69                     & 58.65                    & 20.86                    \\
                         & {\small (54.23 - 93.01)}  & {\small (52.33 - 63.68)} & {\small (14.23 - 21.34)} \\
 \multirow{2}{*}{GMAN-0} & 69.83                     & 49.92                    & 18.90                    \\
                         & {\small (57.66 - 101.62)} & {\small (45.52 - 57.51)} & {\small (15.55 - 22.43)} \\
\midrule
 \multirow{2}{*}{F2U}    & 22.19                     & \textbf{13.38}                    & 14.32                    \\
                         & {\small (16.64 - 29.23)}  & \textbf{\small (10.25 - 15.47)} & {\small (11.32 - 19.27)} \\
 \multirow{2}{*}{F2A}    & \textbf{18.96}                     & 14.53                    & 17.21                    \\
                         & \textbf{\small (16.54 - 19.96)}  & {\small (12.54 - 16.67)} & {\small (13.85 - 25.25)} \\
\bottomrule
\end{tabular}
}
\end{table}
\begin{table}[t]
\caption{\textbf{FID Scores on Fashin MNIST}. The backbone GAN achieved 26.18 (21.08 - 31.10).}
\label{tab:FMNIST}
\centering
\adjustbox{max width=\linewidth}{
\begin{tabular}{lccc}
\toprule
 Methods                 & Non-Ovl          & Mod-Ovl         & Full-Ovl        \\
\midrule
 \multirow{2}{*}{MD-GAN} & 56.09                    & 47.12                    & \textbf{22.04}                    \\
                         & {\small (52.32 - 63.09)} & {\small (45.00 - 51.43)} & \textbf{\small (19.15 - 24.68)} \\
 \multirow{2}{*}{GMAN*}  & 56.79                    & 49.84                    & 27.97                    \\
                         & {\small (53.01 - 58.14)} & {\small (46.91 - 51.52)} & {\small (25.91 - 31.35)} \\
 \multirow{2}{*}{GMAN-0} & 55.21                    & 49.31                    & 29.86                    \\
                         & {\small (51.89 - 57.68)} & {\small (46.00 - 57.81)} & {\small (25.27 - 33.29)} \\
\midrule
 \multirow{2}{*}{F2U}    & 43.07                    & 32.65                    & 36.87                    \\
                         & {\small (34.79 - 56.23)} & {\small (27.94 - 37.21)} & {\small (33.55 - 37.42)} \\
 \multirow{2}{*}{F2A}    & \textbf{37.16}                    & \textbf{29.03}                    & 25.82                    \\
                         & \textbf{\small (26.61 - 37.32)} & \textbf{\small (24.32 - 31.88)} & {\small (24.49 - 29.18)} \\
\bottomrule
\end{tabular}
}
\end{table}
\begin{table}[t]
\caption{\textbf{FID Scores on CINIC-10}. The backbone GAN achieved 23.82 (22.40 - 33.02).}
\label{tab:CINIC10}
\centering
\adjustbox{max width=\linewidth}{
\begin{tabular}{lccc}
\toprule
 Methods                 & Non-Ovl          & Mod-Ovl         & Full-Ovl        \\
\midrule
 \multirow{2}{*}{MD-GAN} & 37.32                    & 34.72                    & \textbf{23.29}                    \\
                         & {\small (35.68 - 45.22)} & {\small (31.21 - 41.54)} & \textbf{\small (20.77 - 30.45)} \\
 \multirow{2}{*}{GMAN*}  & 40.49                    & 36.45                    & 29.38                    \\
                         & {\small (37.74 - 45.61)} & {\small (32.13 - 42.65)} & {\small (26.07 - 35.06)} \\
 \multirow{2}{*}{GMAN-0} & 39.39                    & 35.77                    & 28.28                    \\
                         & {\small (36.46 - 43.22)} & {\small (32.55 - 38.87)} & {\small (26.05 - 37.14)} \\
\midrule
 \multirow{2}{*}{F2U}    & 42.27                    & 31.32                    & 31.31                    \\
                         & {\small (27.55 - 59.76)} & {\small (25.74 - 49.49)} & {\small (27.61 - 44.89)} \\
 \multirow{2}{*}{F2A}    & \textbf{30.60}                    & \textbf{26.06}                    & 27.81                    \\
                         & \textbf{\small (24.36 - 33.62)} & \textbf{\small (22.98 - 36.11)} & {\small (24.83 - 37.58)} \\
\bottomrule
\end{tabular}
}
\end{table}

\begin{figure}
  \begin{center}
    \includegraphics[width=\linewidth]{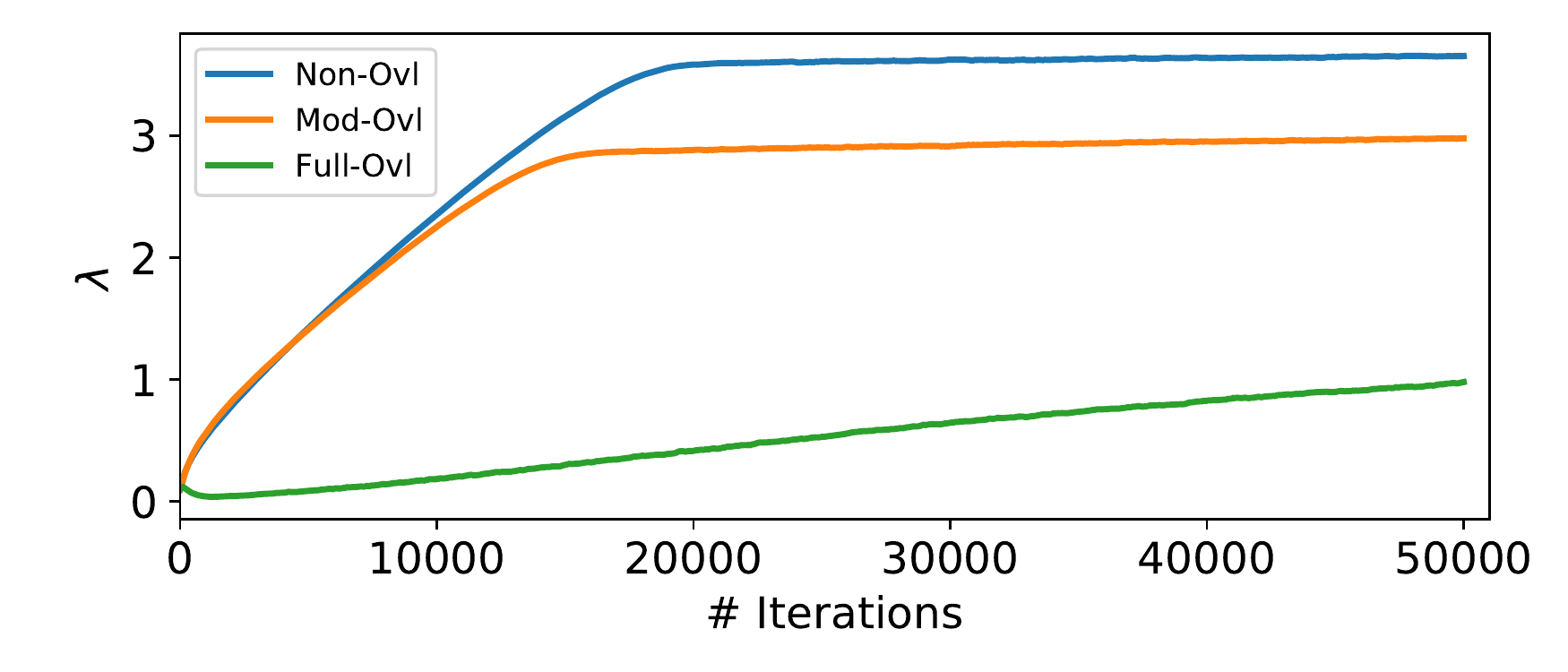}
    \includegraphics[width=\linewidth]{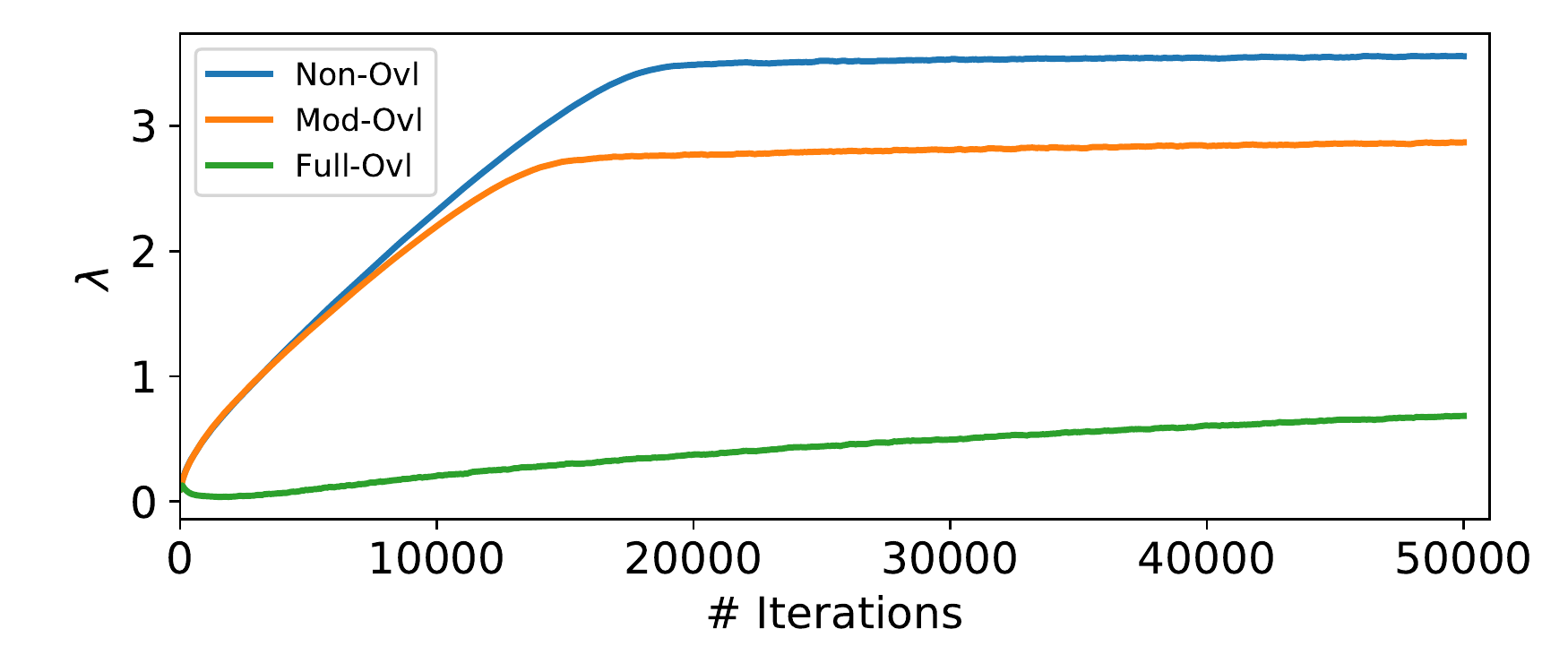}
    \includegraphics[width=\linewidth]{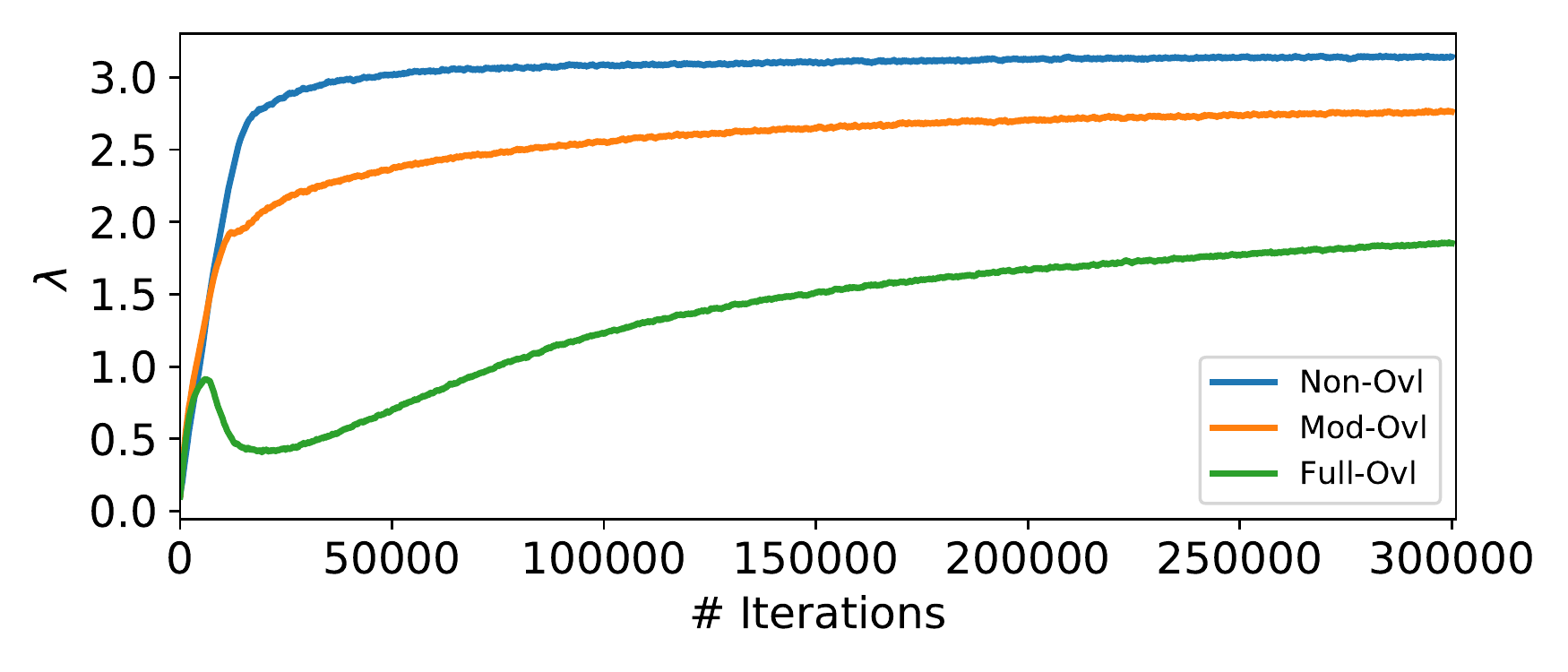}
    \caption{\textbf{Changes of $\lambda$ over Iterations} for MNIST (top), Fashion MNIST (middle), and CINIC-10 (bottom)}
    \label{fig:MNIST_lambda}
  \end{center}
\end{figure}

\paragraph{Effects of Aggregation Parameters}

Table~\ref{tab:abl_3} lists other settings of $\lambda$ for MNIST, including if it was fixed to certain values ($\lambda=0,\;\lambda=3.65$ as the value after saturation under the non-overlapping condition) or was regularized weakly ($\beta=0.01$). Especially under non-overlapping and moderately-overlapping conditions, aggregations with fixed $\lambda$ presented limited performances regardless of how large or small $\lambda$ was, indicating the importance of dynamically updating $\lambda$. The weaker regularization with $\beta=0.01$ instead of $\beta=0.1$ was affected only slightly.


\begin{table}[t]
\caption{\textbf{Effects of Aggregation Parameters on MNIST}}
\label{tab:abl_3}
\centering
\adjustbox{max width=\linewidth}{
\begin{tabular}{lccc}
\toprule
 Methods                 & Non-Ovl          & Mod-Ovl         & Full-Ovl        \\
\midrule
 \multirow{2}{*}{$\lambda=0$ (fixed)}   & 46.76                    & 32.93                    & 17.26                    \\
                                        & {\small (35.93 - 53.85)} & {\small (28.88 - 36.00)} & {\small (14.41 - 22.58)} \\
 \multirow{2}{*}{$\lambda=3.6$ (fixed)} & 22.76                    & 23.35                    & 15.84                    \\
                                        & {\small (14.34 - 23.20)} & {\small (11.97 - 26.81)} & {\small (14.92 - 16.48)} \\
 \multirow{2}{*}{$\beta=0.01$}          & 21.99                    & 13.78                    & 17.01                    \\
                                        & {\small (17.97 - 24.36)} & {\small (12.22 - 18.13)} & {\small (15.88 - 18.43)} \\
\midrule
 \multirow{2}{*}{$\beta=0.1$}           & \textbf{18.96}                    & \textbf{14.53}                    & \textbf{17.21}                    \\
                                        & \textbf{\small (16.54 - 19.96)} & \textbf{\small (12.54 - 16.67)} & \textbf{\small (13.85 - 25.25)} \\
\bottomrule
\end{tabular}
}
\end{table}

\paragraph{Effect of Backbone Model Choices}
Tables~\ref{tab:abl_1} and \ref{tab:abl_2} show MNIST results with several other choices of backbone models, including a standard GAN with binary cross entropy loss (`BCE' in the table) with or without spectral normalization (SN), as well as LSGAN (`MSE' in the table) without SN. Overall, we found that the MSE loss and SN were both important; for both F2U and F2A and for all the conditions, MSE worked better than BCE when combined with SN.

\begin{table}[t]
\caption{\textbf{Effect of Backbone Model Choices on MNIST (F2U)}}
\label{tab:abl_1}
\centering
\adjustbox{max width=\linewidth}{
\begin{tabular}{lcccc}
\toprule
 Loss & SN    & Non-Ovl          & Mod-Ovl         & Full-Ovl        \\
\midrule
 \multirow{2}{*}{BCE} & \multirow{2}{*}{}             & 44.85                    & 21.03                    & 12.63                    \\
                      &                               & {\small (42.31 - 51.28)} & {\small (18.77 - 24.81)} & {\small (11.77 - 18.58)} \\
 \multirow{2}{*}{BCE} & \multirow{2}{*}{$\checkmark$} & 25.85                    & 28.06                    & 29.83                    \\
                      &                               & {\small (21.17 - 35.58)} & {\small (23.28 - 29.62)} & {\small (24.70 - 39.20)} \\
 \multirow{2}{*}{MSE} & \multirow{2}{*}{}             & 44.10                    & 23.16                    & \textbf{11.20}                    \\
                      &                               & {\small (41.88 - 48.36)} & {\small (20.80 - 31.06)} & \textbf{\small (9.52 - 12.61)}  \\
\midrule
 \multirow{2}{*}{MSE} & \multirow{2}{*}{$\checkmark$} & \textbf{22.19}                    & \textbf{13.38}                    & 14.32                    \\
                      &                               & \textbf{\small (16.64 - 29.23)} & \textbf{\small (10.25 - 15.47)} & {\small (11.32 - 19.27)} \\
\bottomrule
\end{tabular}
}
\end{table}

\begin{table}[t]
\caption{\textbf{Effect of Backbone Model Choices on MNIST (F2A)}}
\label{tab:abl_2}
\centering
\adjustbox{max width=\linewidth}{
\begin{tabular}{lcccc}
\toprule
 Loss & SN    & Non-Ovl          & Mod-Ovl         & Full-Ovl        \\
\midrule
 \multirow{2}{*}{BCE} & \multirow{2}{*}{}             & 43.62                    & 21.47                    & 11.89                    \\
                      &                               & {\small (42.59 - 45.35)} & {\small (16.96 - 22.85)} & {\small (11.25 - 16.54)} \\
 \multirow{2}{*}{BCE} & \multirow{2}{*}{$\checkmark$} & 24.04                    & 33.65                    & 28.02                    \\
                      &                               & {\small (20.33 - 26.42)} & {\small (26.93 - 40.08)} & {\small (23.40 - 29.34)} \\
 \multirow{2}{*}{MSE} & \multirow{2}{*}{}             & 74.65                    & 30.19                    & \textbf{11.72}                    \\
                      &                               & {\small (52.63 - 86.38)} & {\small (23.25 - 37.95)} & \textbf{\small (8.96 - 13.24)}  \\
\midrule
 \multirow{2}{*}{MSE} & \multirow{2}{*}{$\checkmark$} & \textbf{18.96}                    & \textbf{14.53}                    & 17.21                    \\
                      &                               & \textbf{\small (16.54 - 19.96)} & \textbf{\small (12.54 - 16.67)} & {\small (13.85 - 25.25)} \\
\bottomrule
\end{tabular}
}
\end{table}

\paragraph{Effect of the Number of Clients}
We also tested how performances changed when the number of clients $N$ became large: $N=10,\;N=20$. For $N=10$, we split MNIST into ten subsets such that each subset involved images of the only single digit. For $N=20$, we further divided each subset obtained in $N=10$ randomly into two subsets of the same size. Figure~\ref{fig:fid_N_MNIST} shows the median FID scores. As a reference, we also present $N=5$ under the non-overlapping condition in the figure. Both F2U and F2A clearly outperformed the other methods even when $N$ was large.

\begin{figure}[t]
  \begin{center}
    \includegraphics[width=\linewidth]{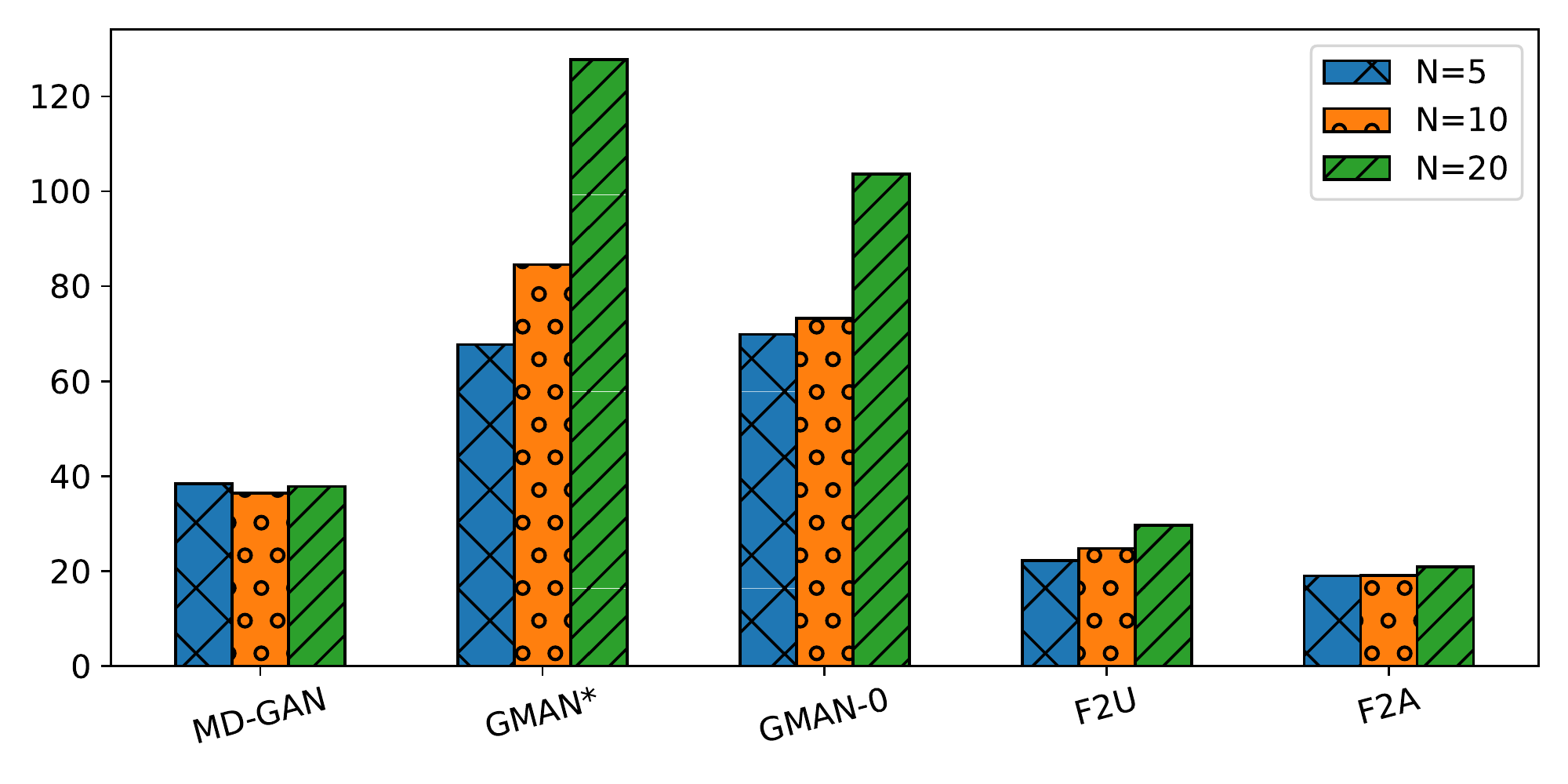}
    \caption{\textbf{Effects of N.} Median FID scores on MNIST across multiple hyperparameter combinations.}
    \label{fig:fid_N_MNIST}
  \end{center}
\end{figure}

\subsection{\textcolor{black}{Visualizing Generator's Distributions}}
\label{subsec:visualize}
\textcolor{black}{Finally, we conducted an extra experiment to visualize how well each method was able to learn non-iid data distributions with simpler examples. As shown in Figure~\ref{fig:toy}, we considered a problem of learning Gaussian mixtures in 1D an 2D spaces, where each client data was drawn from a distinct Gaussian distribution. We used a simple multi-layer perceptron for both of a generator and discriminators. We confirmed that the generator's distributions for F2U and F2A were fit to multiple client data distributions successfully, whereas the other methods completely failed to work.}

\begin{figure}[t]
  \begin{center}
    \includegraphics[width=\linewidth]{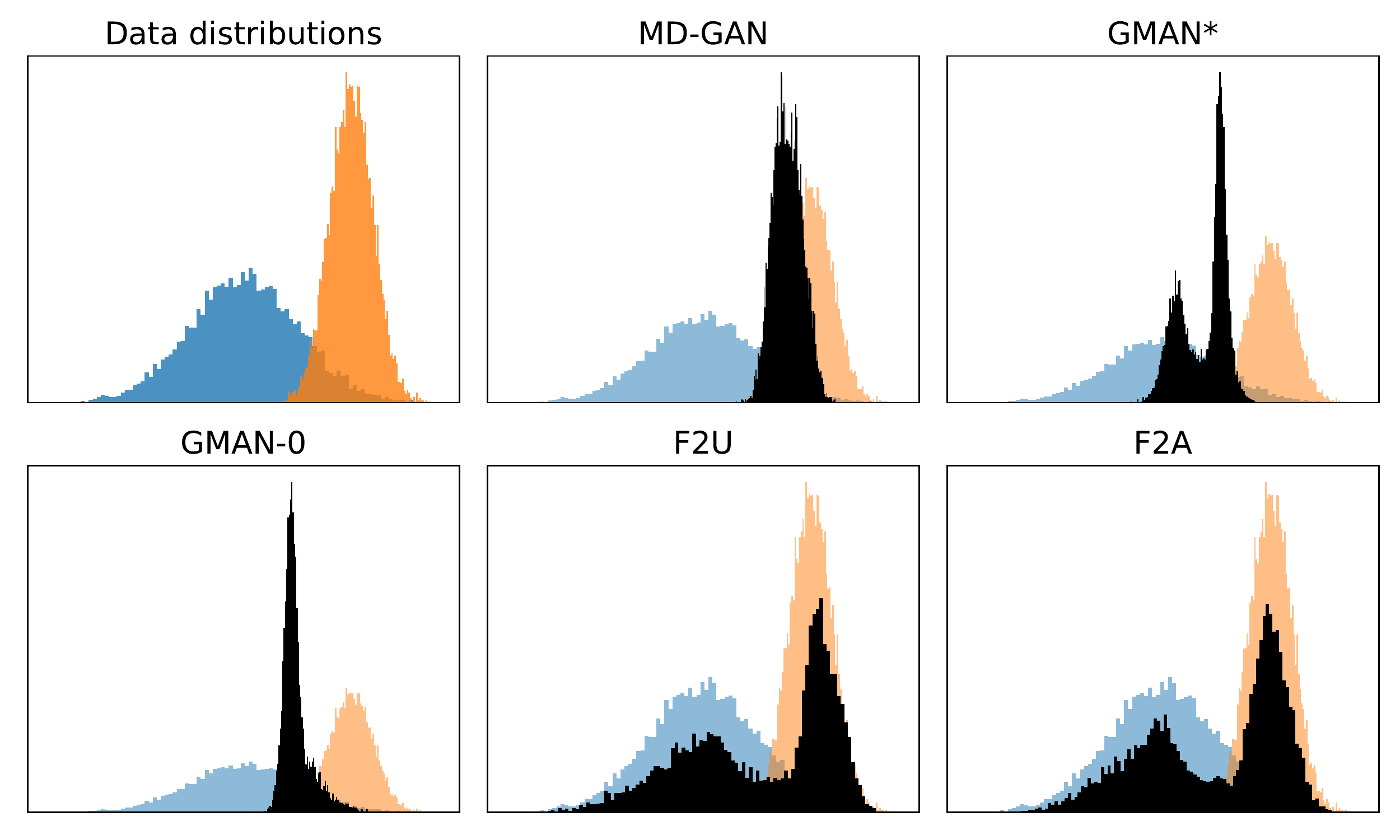}
    \includegraphics[width=\linewidth]{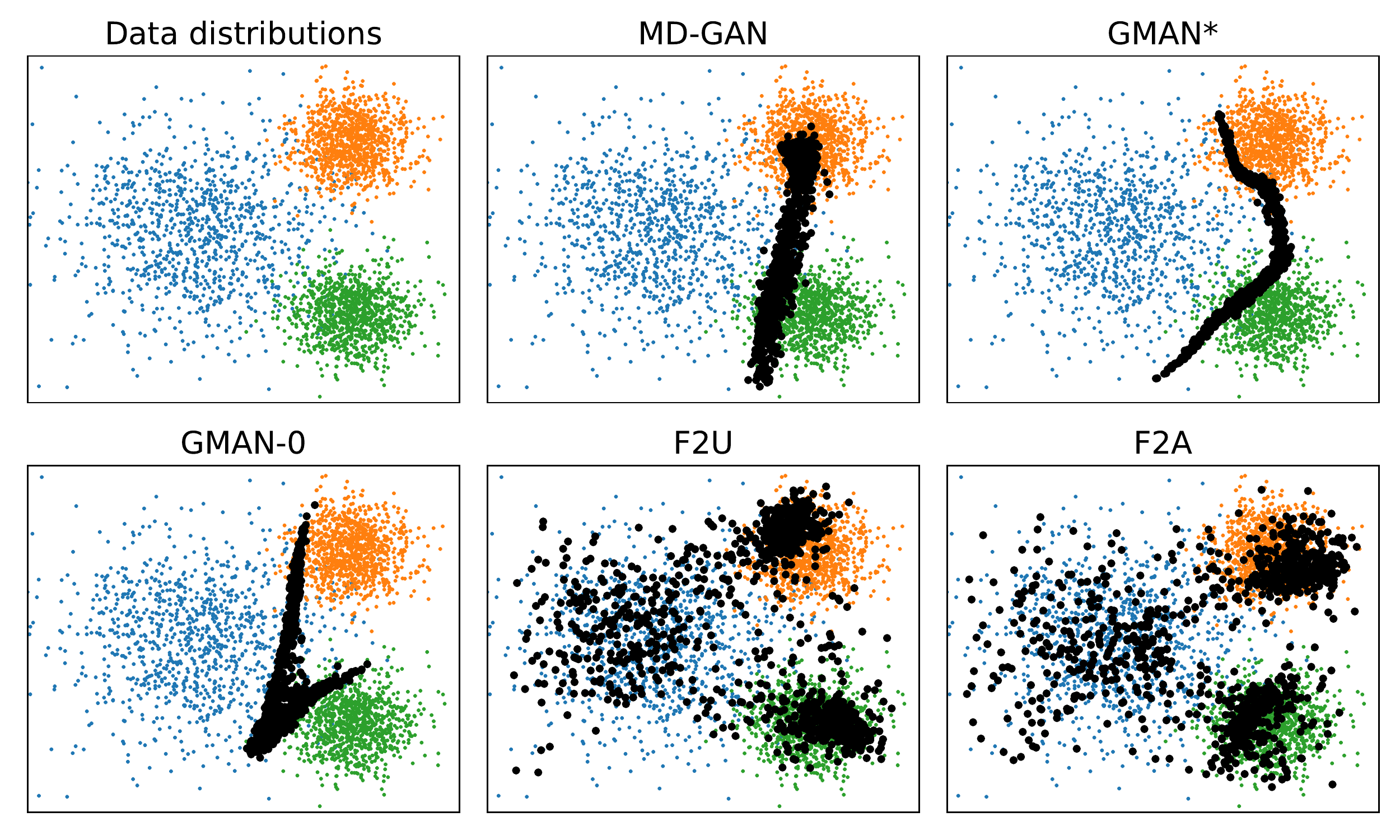}
    \caption{\textbf{Learning Gaussian Mixtures.} Blue/orange/green points describe data owned by different clients, while black points are sampled from learned generators.}
    \label{fig:toy}
  \end{center}
\end{figure}

\subsection{Limitations}
Our work has several limitations. i) Our empirical evaluation is based on a simulation in a single workstation. Practical implementations of the proposed approach as well as other baselines will come with problems of communication, security, and scalability for a large number of clients, as discussed in \cite{Mcmahan2017}. ii) Our approach by itself is not designed to resolve common problems observed in GAN training, such as mode collapse and theoretical guarantee for convergence. We still require additional contributions to make GANs perform well in practice. iii) The current formulation can be applied only to standard GANs with a generator and a discriminator. One interesting extension for future work is to deal with various architectures such as conditional GANs~\cite{Isola2017,Mirza2014,Miyato2018b,Odena2017,Reed2016} and GANs with multiple generators~\cite{Chavdarova2018,Choi2018,Ghosh2018,Liu2016,Zhu2017}.

\section{Conclusion}
We addressed the problem of learning GANs in a decentralized fashion from multiple non-iid data collections and presented new approaches called Forgiver-First Update (F2U) and Forgiver-First Aggregation (F2A). We hope that our work has raised a new challenge of decentralized deep learning, \ie, unsupervised decentralized learning from non-iid data, and will also impact a variety of real-world applications such as anomaly detection using decentralized medical data and learning image compression models using photo collections stored in smartphones.

\appendix

\section{On the Convexity of $f$}
To show the global optimality of $p\sub{g}=p\sub{max}$ with LSGANs in Theorem 2, we defined the following function $f$ in Eq. (13). 
\begin{equation}
   f(x)=\frac{(x+1)\alpha^2 x^2}{(1 + \alpha x)^2} - \frac{2\alpha^2}{(1+\alpha)^2},
\end{equation}
where $\alpha=\frac{1}{Z}=(\int_x\max_ip_i(x)\textrm{d}x)^{-1}\leq 1$ and the equality here holds if and only if $p_1=p_2=\dots=p_N$. This function needs to be convex at least for $x\geq 0$ to be used with the $f$-divergence. To show its convexity, we calculate the second derivative of $f$:
\begin{equation}
   f^{\prime\prime}(x)= \frac{2\alpha^2(1 + (3-2\alpha)x)}{(1+\alpha x)^4}.
\end{equation}
Since $\alpha \leq 1$, $f^{\prime\prime}(x)\geq 0$ if $x\geq 0$, namely, $f$ is convex for $x\geq 0$.

\section{Global Optimality of $p\sub{g}=p\sub{max}$ with the Standard GAN}
In addition to our main theoretical results that show the global optimality of $p\sub{g}=p\sub{max}$ with LSGANs, we here prove that the same global optimum can be achieved also for the standard GAN using the binary cross-entropy loss.

As proven in \cite{Goodfellow2014}, the optimal discriminator trained from data-generating distribution $p_i(x)$ with generator's distribution $p\sub{g}(x)$ fixed is given as follows:
\begin{eqnarray}
  D^*_i(x) = \frac{p_i(x)}{p_i(x) + p\sub{g}(x)}.
  \label{opt_D}
\end{eqnarray}
As we showed in Lemma 1 of the main paper, $D^*\sub{max}(x)=\max_i D^*_i(x)$ can then be regarded as the optimal discriminator trained from $p\sub{max}(x)$, namely,
\begin{equation}
    D^*\sub{max}(x)=\frac{p\sub{max}(x)}{p\sub{max}(x) +\alpha p\sub{g}(x)},
\end{equation}
where $\alpha=\frac{1}{Z}$ is a positive constant. On the other hand, the objective function for the generator in \cite{Goodfellow2014}, given $D^*\sub{max}(x)$, can be reformulated as:
\small
\begin{eqnarray}
\mathcal{L}_G=\mathbb{E}_{x\sim p\sub{max}}\left[\log(D^*_{max}(x))\right] + \mathbb{E}_{x\sim p\sub{g}}\left[\log(1 - D^*_{max}(x))\right]\\
  =\int_x\left[p\sub{max}(x) \log\left(\frac{p\sub{max}(x)}{p\sub{max}(x) + \alpha p\sub{g}(x)}\right) \right.\\
  \left.+ p\sub{g}(x)\log\left(\frac{\alpha p\sub{g}(x)}{p\sub{max}(x) + \alpha p\sub{g}(x)}\right)\right]\textrm{d}x.
  \label{problem_G}
\end{eqnarray}
\normalsize

Now, consider the following continuous function $f$:
\begin{equation}
  f(x) = -(1 + x) \log(1 + \alpha x) + x \log(x) + 2\log(1 + \alpha)\ .
\end{equation}
where $f(1)=0$ and its second derivative is:
\begin{equation}
  f^{\prime\prime}(x) = \frac{1 + \alpha^2x}{x(1 + \alpha x)^2}.
\end{equation}
Since $f^{\prime\prime}(x) \geq 0$ if $x\geq 0$, the function $f$ is convex for $x\geq 0$. We introduce the $f$-divergence with this function as follows:
\small
\begin{eqnarray}
  D_{f}(p\mid\mid q) &=&\int_x q(x)f\left(\frac{p(x)}{q(x)}\right) \textrm{d}x\\
  &=& \int_x\left\{q(x)\log\left(\frac{q(x)}{q(x) + \alpha p(x)}\right)\right. \\
  &&\left.+ p(x)\log\left(\frac{\alpha p(x)}{q(x) + \alpha p(x)}\right)\right\}\textrm{d}x + C,
\end{eqnarray}
\normalsize
where $C=2\log(1 + \alpha) - \log(\alpha)$ is a constant. With $D_f$, $\mathcal{L}_G$ in Eq.~(\ref{problem_G}) can be rearranged:
\begin{equation}
    \mathcal{L}_G=D_f(p\sub{g}\mid\mid p\sub{max}) - C,
\end{equation}
which reaches its global minimum if and only if $p\sub{g}=p\sub{max}$.

\section{Implementation Details}
This section presents implementation details of the backbone GANs used in our experiments.

\paragraph{MNIST and Fashion MNIST}
The architecture of the generator was designed as follows. A 128-dimensional noise vector drawn from the normal distribution $\mathcal{N}(0, I)$ was first fed to a fully connected layer with $256\times7\times7$ channels and activated with ReLU~\cite{Nair2010}, which was then reshaped into a feature map sized $7\times 7$ and with $256$ channels. This feature map was then deconvoluted using two consecutive 2D deconvolution layers with the kernel size of $4$, the stride of $2$, and the channels of $128$ (first layer) and $64$ (second layer), both of which were batch-normalized with the momentum of $0.1$ and activated with ReLU. Finally, one more deconvolution layer with the kernel size of $3$, the stride of $1$, and the channel of $1$, which was activated by the hyperbolic tangent, was applied to obtain gray-scale images of the size $28\times 28$. The discriminator that received gray-scale images with the size of $28\times 28$ consisted of four consecutive convolution layers, which all had the kernel size of $3$, the stride of $2$, and the channels of $[32, 64, 128, 256]$, followed by spectral normalization~\cite{Miyato2018} and LeakyReLU activation ($\alpha=0.2)$~\cite{Xu2015}. Zero-padding was applied before the second convolutional filter to down-scale feature maps properly in the subsequent convolutions. Finally, the feature maps were flattened and fed into a fully connected layer with one-dimensional output followed by spectral normalization and linear activation.

\paragraph{CINIC-10}
Similar to the architecture shown above, the generator first fed a 128-dimensional noise vector drawn from $\mathcal{N}(0, I)$ to a fully connected layer with $512\times 4\times 4$ channels and the ReLU activation. The output was reshaped into a feature map sized $4\times 4$ and with $512$ channels, and then fed to three consecutive 2D deconvolution layers with the kernel size of $4$, the stride of $2$, and the channels of $[256, 128, 64]$. Each deconvolution layer was followed by the batch normalization with the momentum of $0.1$ and the ReLU activation. One more deconvolution layer with the kernel size of $3$, the stride of $1$, and the channel of $3$, which was activated by the hyperbolic tangent, was applied finally to obtain colored images of the size $32\times 32$. The discriminator consisted of five convolution layers with the channels of $[64, 64, 128, 128, 256]$, the kernel size of $[3, 4, 3, 4, 4]$, and the stride of $[1, 2, 1, 2, 2]$, respectively. Each convolution layer was followed by spectral normalization and leaky ReLU with $\alpha=0.1$, and the output was flattened and fed to a fully connected layer with a single channel with spectral normalization and linear activation.

\balance
\section{Qualitative Results}
Finally, we show some examples of generated images in Figure~\ref{fig:qual}. We found i) lower quality images with MD-GAN; and ii) biased outputs (\eg, many `1's generated) with GMAN* and GMAN-0, while iii) F2U and F2A did not provide such major issues. 

\begin{figure*}[t]
\begin{minipage}{0.33\linewidth}
    \includegraphics[width=\linewidth]{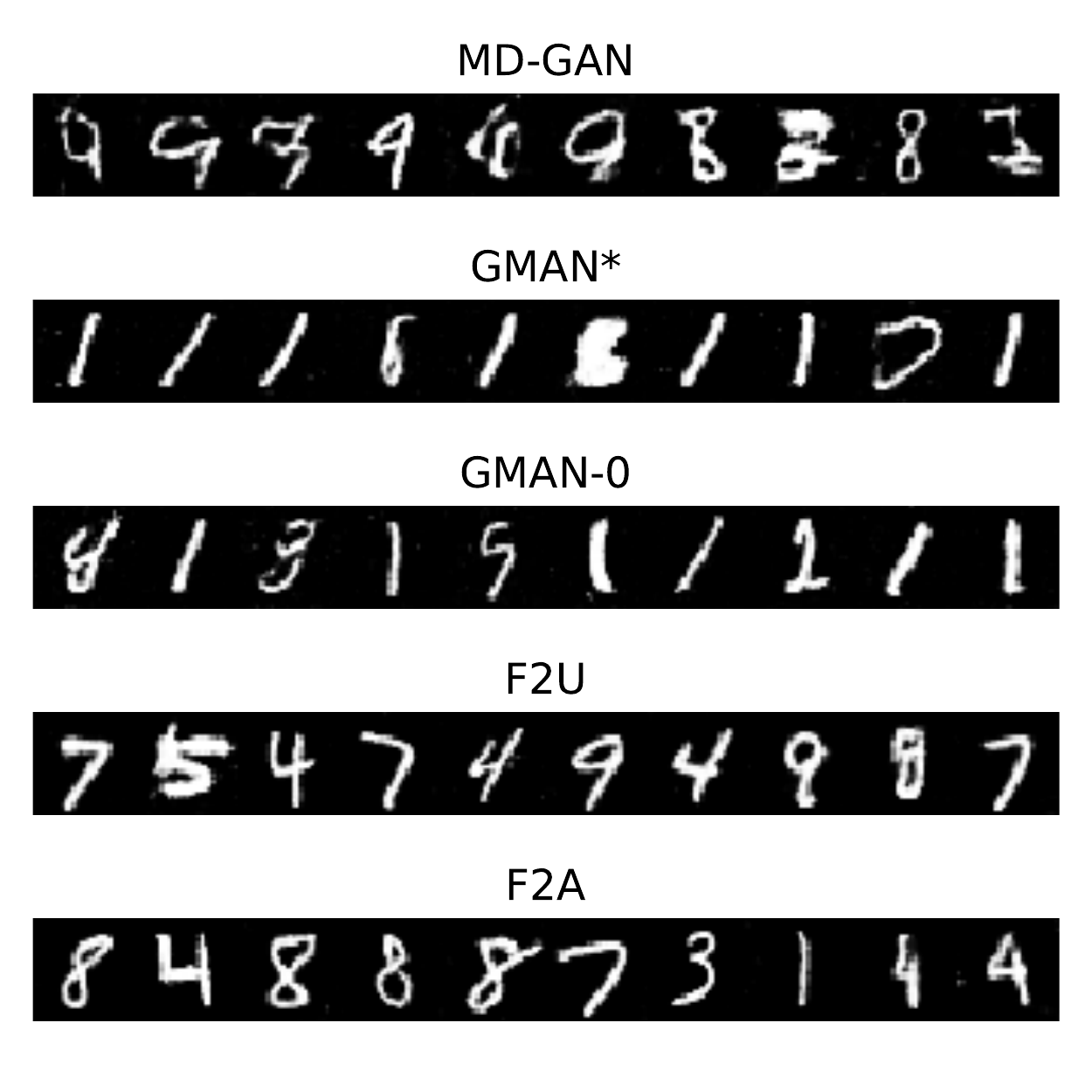}
    \subcaption{MNIST (Non-Ovl)}
    \end{minipage}
\begin{minipage}{0.33\linewidth}
    \includegraphics[width=\linewidth]{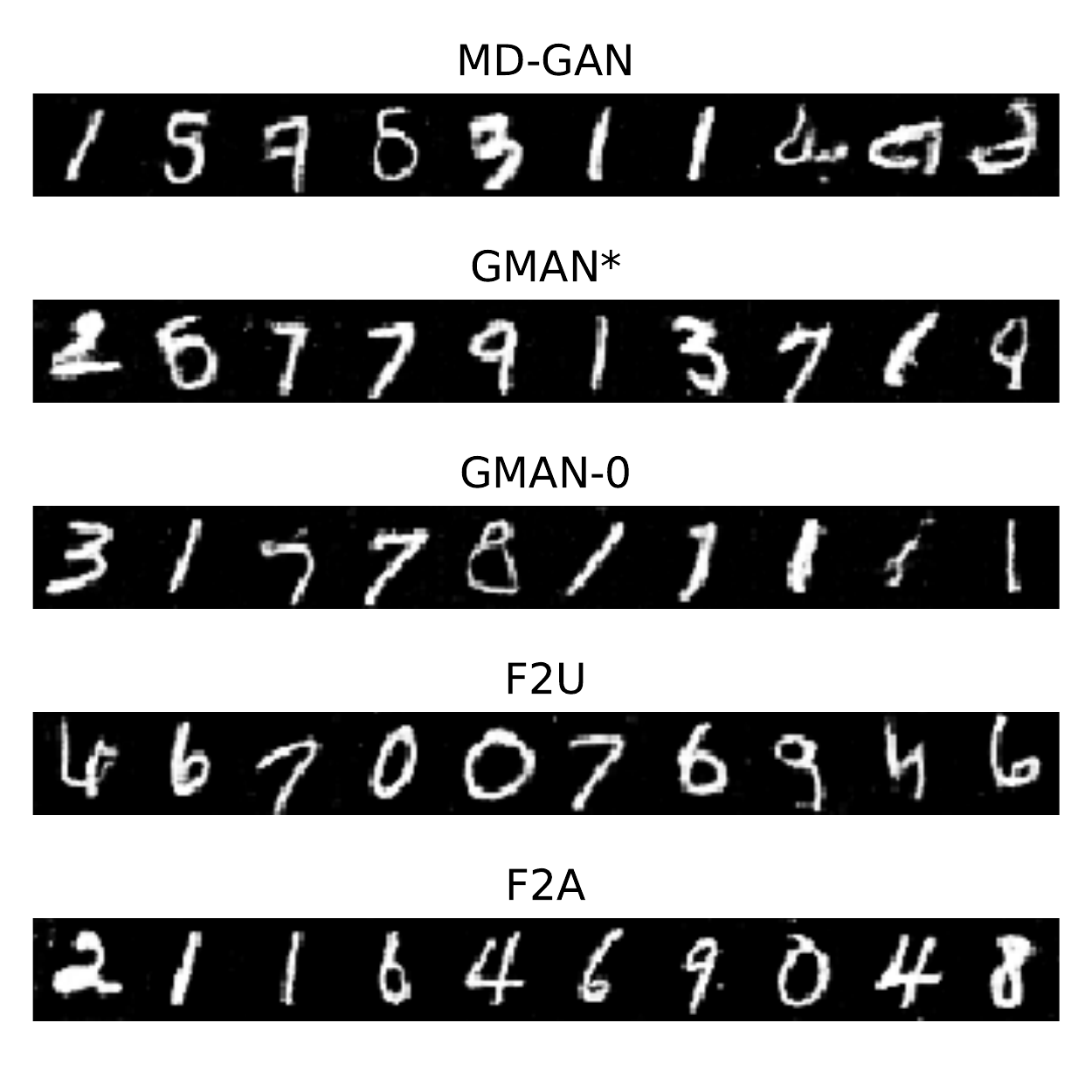}
    \subcaption{MNIST (Mod-Ovl)}
    \end{minipage}
\begin{minipage}{0.33\linewidth}
    \includegraphics[width=\linewidth]{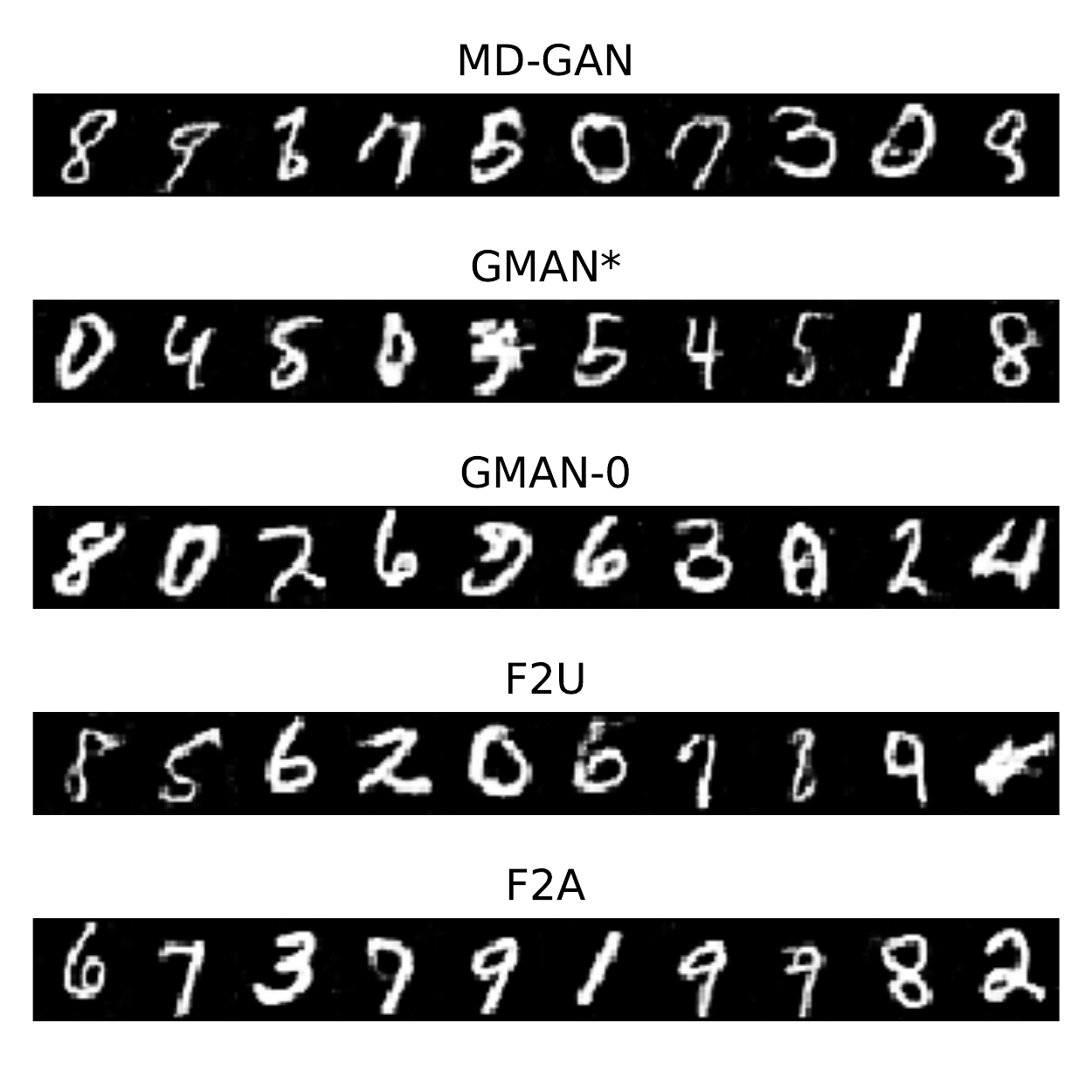}
    \subcaption{MNIST (Full-Ovl)}
    \end{minipage}
\begin{minipage}{0.33\linewidth}
    \includegraphics[width=\linewidth]{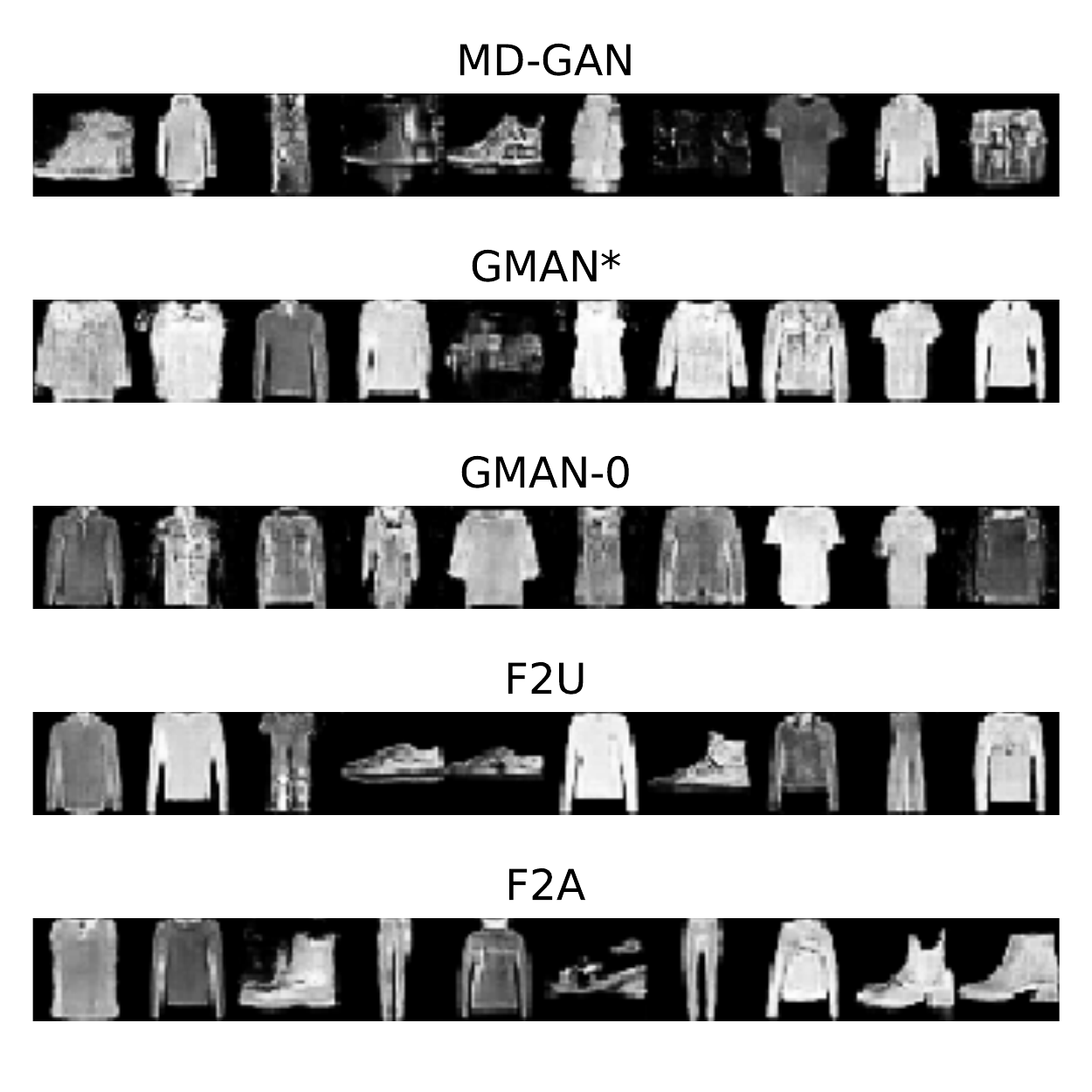}
    \subcaption{Fashion MNIST (Non-Ovl)}
    \end{minipage}
\begin{minipage}{0.33\linewidth}
    \includegraphics[width=\linewidth]{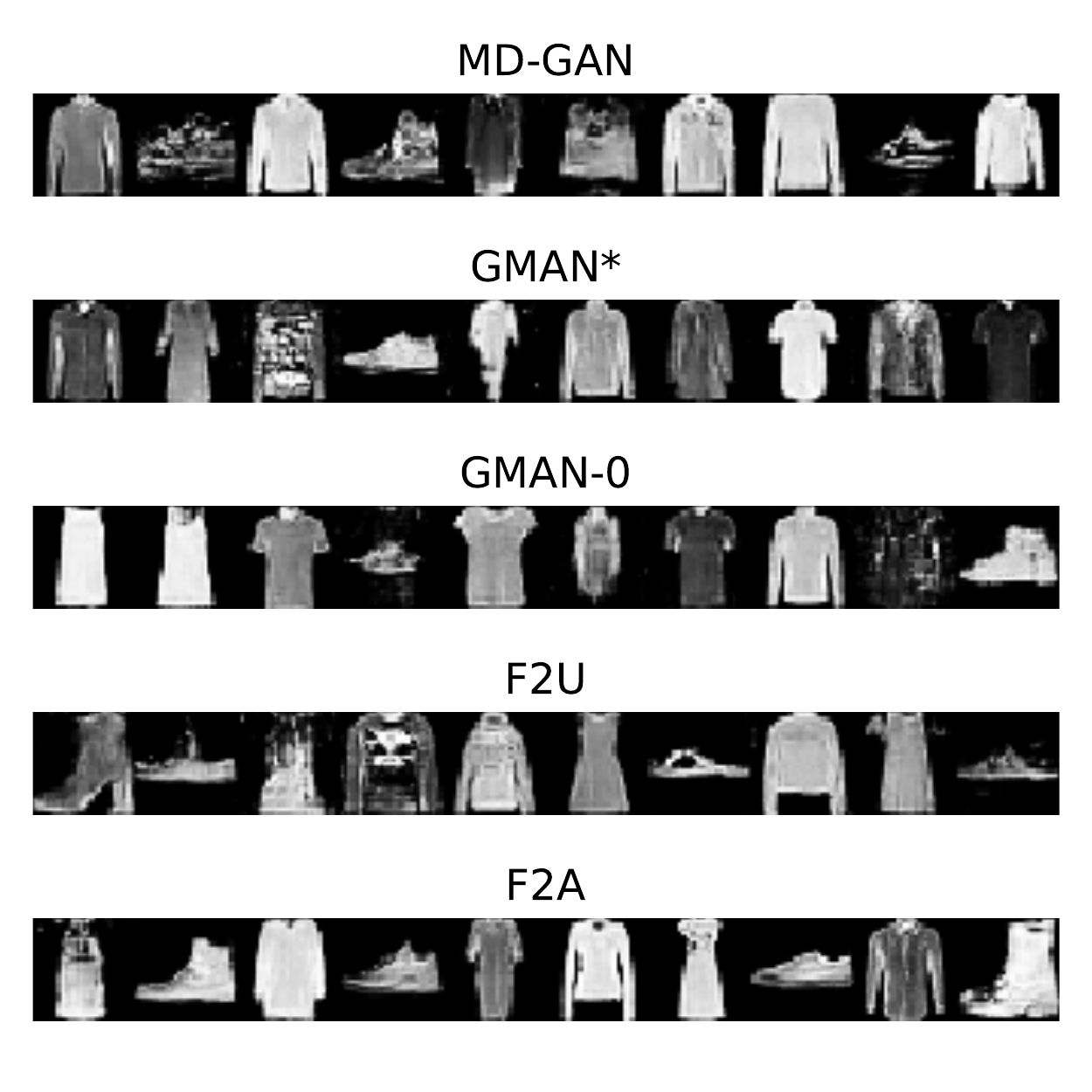}
    \subcaption{Fashion MNIST (Mod-Ovl)}
    \end{minipage}
\begin{minipage}{0.33\linewidth}
    \includegraphics[width=\linewidth]{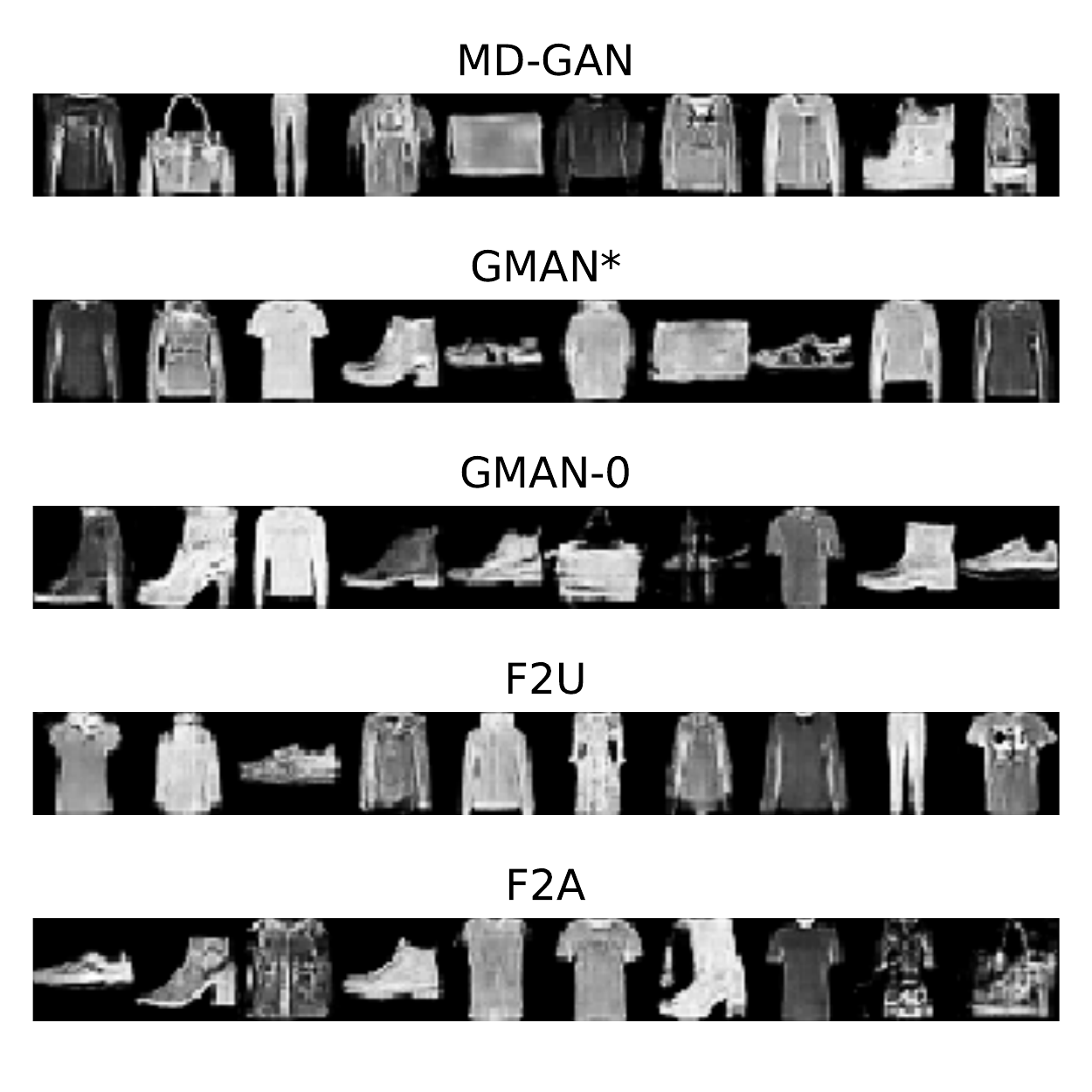}
    \subcaption{Fashion MNIST (Full-Ovl)}
    \end{minipage}
\begin{minipage}{0.33\linewidth}
    \includegraphics[width=\linewidth]{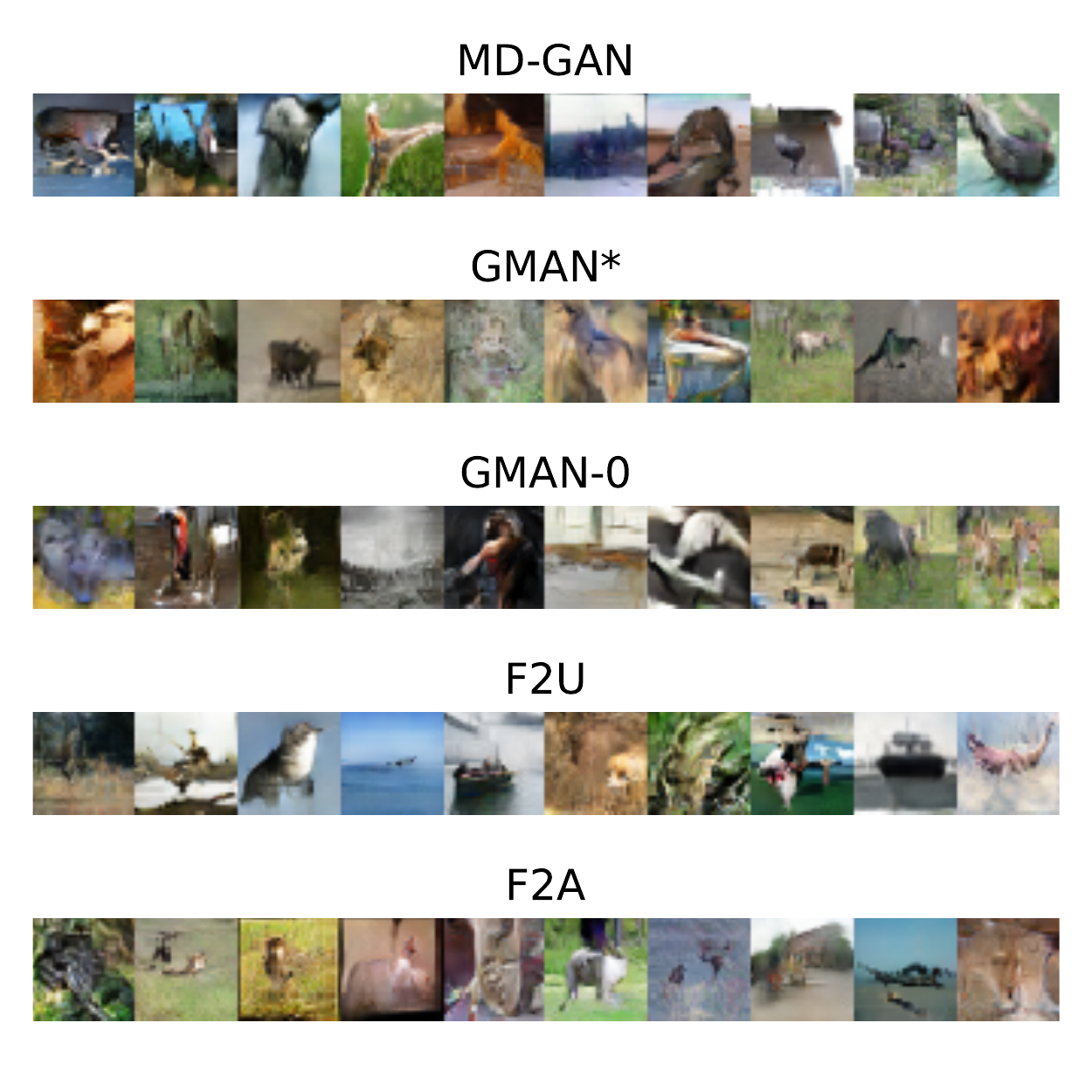}
    \subcaption{CINIC-10 (Non-Ovl)}
    \end{minipage}
\begin{minipage}{0.33\linewidth}
    \includegraphics[width=\linewidth]{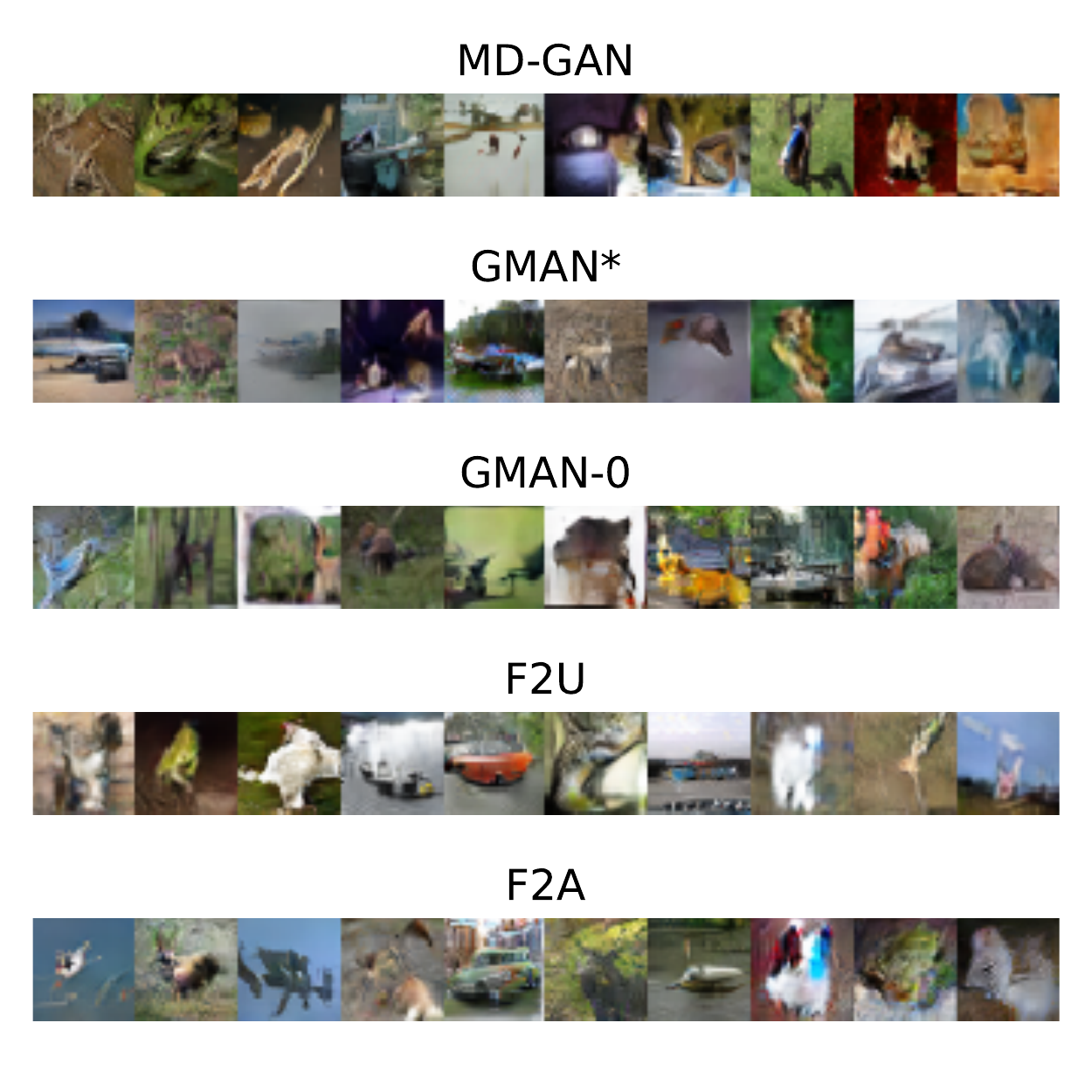}
    \subcaption{CINIC-10 (Mod-Ovl)}
    \end{minipage}
\begin{minipage}{0.33\linewidth}
    \includegraphics[width=\linewidth]{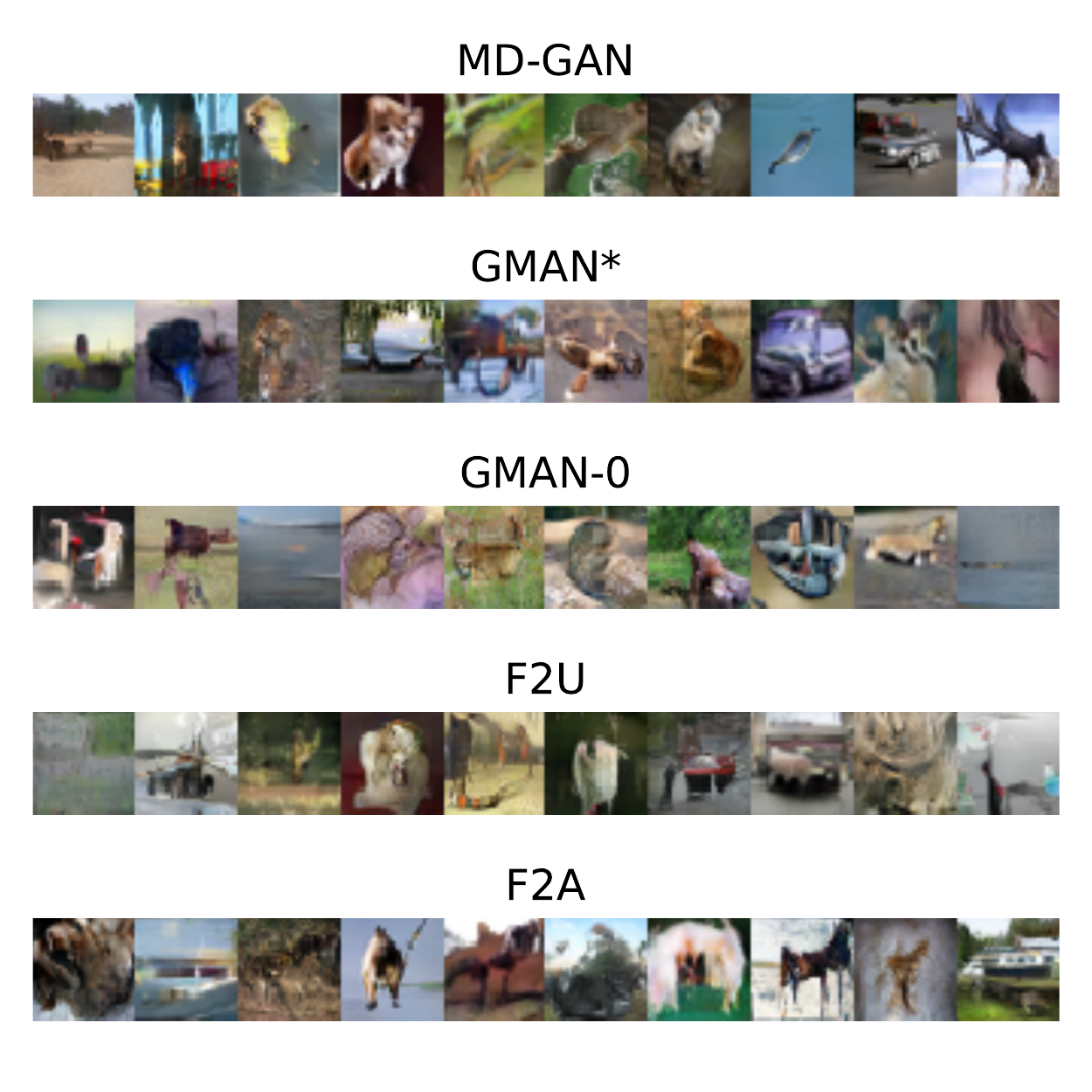}
    \subcaption{CINIC-10 (Full-Ovl)}
    \end{minipage}
    \caption{Qualitative Results}
    \label{fig:qual}
\end{figure*}

{\small
\bibliographystyle{ieee_fullname}

}

\end{document}